\documentclass{article}

\usepackage{microtype}
\usepackage{graphicx}
\usepackage{subcaption}
\usepackage{booktabs} %
\usepackage{afterpage} %
\usepackage{multirow} %
\usepackage{placeins}

\usepackage{hyperref}

\usepackage[preprint]{icml2026}

\usepackage{amsmath}
\usepackage{amssymb}
\usepackage{mathtools}
\usepackage{amsthm}

\usepackage[capitalize,noabbrev]{cleveref}

\usepackage{amsbsy}
\usepackage{esint}
\usepackage{bbm}
\usepackage{nicefrac}
\usepackage{tikz}
\usetikzlibrary{arrows.meta, positioning}

\let\Pr\undefined

\newcommand{\E}{\mathbb{E}}
\newcommand{\mcD}{\mathcal{D}}

\newcommand{\Pr}{\ensuremath{\mathrm{P}}}

\def\ddefloop#1{\ifx\ddefloop#1\else\ddef{#1}\expandafter\ddefloop\fi}
\def\ddef#1{\expandafter\def\csname bb#1\endcsname{\ensuremath{\mathbb{#1}}}}
\ddefloop ABCDEFGHIJKLMNOPQRSTUVWXYZ\ddefloop
\def\ddef#1{\expandafter\def\csname b#1\endcsname{\ensuremath{\mathbf{#1}}}}
\ddefloop ABCDEFGHIJKLMNOPQRSTUVWXYZ\ddefloop
\def\ddef#1{\expandafter\def\csname c#1\endcsname{\ensuremath{\mathcal{#1}}}}
\ddefloop ABCDEFGHIJKLMNOPQRSTUVWXYZ\ddefloop

\theoremstyle{plain}
\newtheorem{theorem}{Theorem}[section]
\newtheorem{proposition}[theorem]{Proposition}

\theoremstyle{definition}

\newtheorem{assumption}[theorem]{Assumption}
\theoremstyle{remark}

\newtheorem{observation}[theorem]{Observation}

\icmltitlerunning{Robust Post-Training for Generative Recommenders: Why Exponential Reward-Weighted SFT Outperforms RLHF}

\begin{document}

\twocolumn[
  \icmltitle{Robust Post-Training for Generative Recommenders: Why Exponential Reward-Weighted SFT Outperforms RLHF}

  \icmlsetsymbol{equal}{*}

  \begin{icmlauthorlist}
    \icmlauthor{Keertana Chidambaram}{yyy}
    \icmlauthor{Sanath Kumar Krishnamurthy}{equal,comp1}
    \icmlauthor{Qiuling Xu}{equal,comp2}
    \icmlauthor{Ko-Jen Hsiao}{comp2}
    \icmlauthor{Moumita Bhattacharya}{comp2}
  \end{icmlauthorlist}

  \icmlaffiliation{yyy}{Department of Management Science \& Engineering, Stanford University, CA, USA}
  \icmlaffiliation{comp1}{Meta, CA, USA}
  \icmlaffiliation{comp2}{Netflix Research, CA, USA}

  \icmlcorrespondingauthor{Keertana Chidambaram}{vck@stanford.edu}

  \icmlkeywords{Generative Recommenders, Post-training}
  \vskip 0.3in
]

\printAffiliationsAndNotice{\icmlEqualContribution Work done while KC was an intern at Netflix Research.}

\newcommand{\qiuling}[1]{%
    \textcolor{teal}{\textbf{Qiuling}}\textcolor{red}{\textit{[#1]}}%
}

\begin{abstract}
  Aligning generative recommender systems to user preferences via post-training is critical for closing the gap between next-item prediction and actual recommendation quality. Existing post-training methods are ill-suited for production-scale systems: RLHF methods reward hack due to noisy user feedback and unreliable reward models, offline RL alternatives require propensity scores that are unavailable, and online interaction is infeasible. We identify exponential reward-weighted SFT with weights $w = \exp(r/\lambda)$ as uniquely suited to this setting, and provide the theoretical and empirical foundations that explain why. By optimizing directly on observed rewards without querying a learned reward model, the method is immune to reward hacking, requires no propensity scores, and is fully offline. We prove the first policy improvement guarantees for this setting under noisy rewards, showing that the gap scales only logarithmically with catalog size and remains informative even for large item catalogs. Crucially, we show that temperature $\lambda$ explicitly and quantifiably controls the robustness-improvement tradeoff, providing practitioners with a single interpretable regularization hyperparameter with theoretical grounding. Experiments on three open-source and one proprietary dataset against four baselines confirm that exponential reward weighting is simple, scalable, and consistently outperforms RLHF-based alternatives.

\end{abstract}

\section{Introduction}
The success of transformer architectures has led to the rise of generative recommenders. This paradigm treats recommendations as a sequential generation problem where recommendations are generated sequentially based on the user's interaction history \cite{yang2025gr, deldjoo2024review, wu2024survey}, analogous to how Large Language Models (LLMs) predict the next token given a sequence of text \cite{rajput2023recommender, zhai2024actions, kang2018self, deng2025onerec}. Vanilla generative recommender architectures such as SASRec \cite{kang2018self}, HSTU \cite{zhai2024actions}, and OneRec \cite{deng2025onerec} model user behavior through large-scale behavior cloning. However, purely imitating users through behavior cloning leads to indiscriminate mimicking, where the model learns high-value engagements (like items they enjoyed) and low-value engagements (like accidental clicks or click-bait interactions) with equal importance. 

Motivated by recent advances in Reinforcement Learning from Human Feedback (RLHF) \cite{ouyang2022training, stiennon2020learning}, this paper explores post-training alignment of generative recommenders with rich user feedback. The core insight from RLHF in language models is that post-training optimization with preference data can significantly improve model quality beyond what behavior cloning alone achieves. This paradigm is particularly promising for recommendation systems, where massive organic user feedback is naturally available at scale—including implicit signals such as watch time, completion rates, and re-engagement patterns, as well as explicit signals such as ratings and reviews. By leveraging this abundant feedback to move beyond pure behavior cloning, we can potentially train recommenders that distinguish high-value from low-value engagements and explore beyond the confines of observed user trajectories. However, adapting RLHF techniques from language models to generative recommenders introduces unique challenges due to the following factors.

\begin{itemize}
    \item \textbf{Reward Model Unreliability.} Item representations in generative recommenders are learned purely from behavioral data with no semantic grounding. Since users interact with only a small fraction of the catalog \cite{dulac2015deep}, a reward model must extrapolate over the vast majority of items from sparse supervision, a generalization problem that grows severe at scale. During policy optimization, the policy exploits these extrapolation errors, systematically selecting items for which the reward model is overoptimistic rather than those that truly maximize user satisfaction.

    \item \textbf{Offline Learning Constraints.} In industrial settings the learning dataset is pre-collected and static, making interactive feedback loops infeasible. RLHF \cite{ouyang2022training, stiennon2020learning} requires a reward model as a simulator. DPO \cite{rafailov2023direct} avoids this but requires binary preference pairs, whereas recommendation feedback is scalar. In practice, overlapping interactions covering the full catalog are too sparse to construct such pairs without again relying on a learned reward model.

    \item \textbf{Lack of Logging Policy.} The offline dataset exhibits selection bias, as rewards are only observed for actions taken by the logging policy. While Inverse Propensity Scoring (IPS) could theoretically correct this, the logging policy is often too complex \cite{covington2016deep} and inaccessible to estimate \cite{liang2022local}, and IPS weights suffer from extreme variance \cite{swaminathan2015batch, dudik2011doubly}, particularly when the logging policy is nearly deterministic.
\end{itemize}

To overcome these challenges, we propose Exponential Reward-Weighted SFT (Exp-RSFT), which weights training examples by $\exp(r/\lambda)$ using only observed rewards. By never querying a learned reward model, the method sidesteps reward model unreliability at its root. It cannot reward hack, requires no propensity scores, and is fully offline. While exponential reward weighting has appeared in prior work \cite{peters2007using, wang2018exponentially, liang2022local}, its suitability for large-scale generative recommenders has not been theoretically justified nor empirically validated at production scale. We make the following contributions:
\begin{itemize}
    \item \textbf{Reward models fail catastrophically in the generative recommender setting.} Reward models must extrapolate over a vast item catalog from sparse observations, making them unreliable surrogates for true user preferences. We show empirically that the learned reward model fails to outperform naive item-mean prediction, yet PPO and DPO over-optimize for it, collapsing catastrophically on true recommendation metrics.
    \item \textbf{Temperature $\lambda$ provably controls the robustness-improvement tradeoff under noisy rewards.} Observed rewards are inherently noisy even without a learned reward model. We prove policy improvement guarantees under noisy rewards with only logarithmic dependence on catalog size, and a closed-form characterization of the robustness-improvement tradeoff controlled by $\lambda$. The $\lambda$ sweep produces a consistent inverted-U performance curve across all datasets, empirically validating the theory.
    \item \textbf{Exp-RSFT consistently outperforms all baselines across four datasets.} Exp-RSFT consistently outperforms all baselines across three public benchmarks and a large-scale proprietary dataset. Gains over PPO and DPO stem from avoiding the reward model entirely; gains over behavior cloning and linear reward-weighted SFT isolate $\lambda$ as the key ingredient, which provides explicit control over re-ranking aggressiveness and robustness to noise in a single interpretable parameter.
\end{itemize}

\section{Related Literature}

\begin{figure}[t]
  \vskip 0.2in
  \begin{center}    \centerline{\includegraphics[width=\columnwidth]{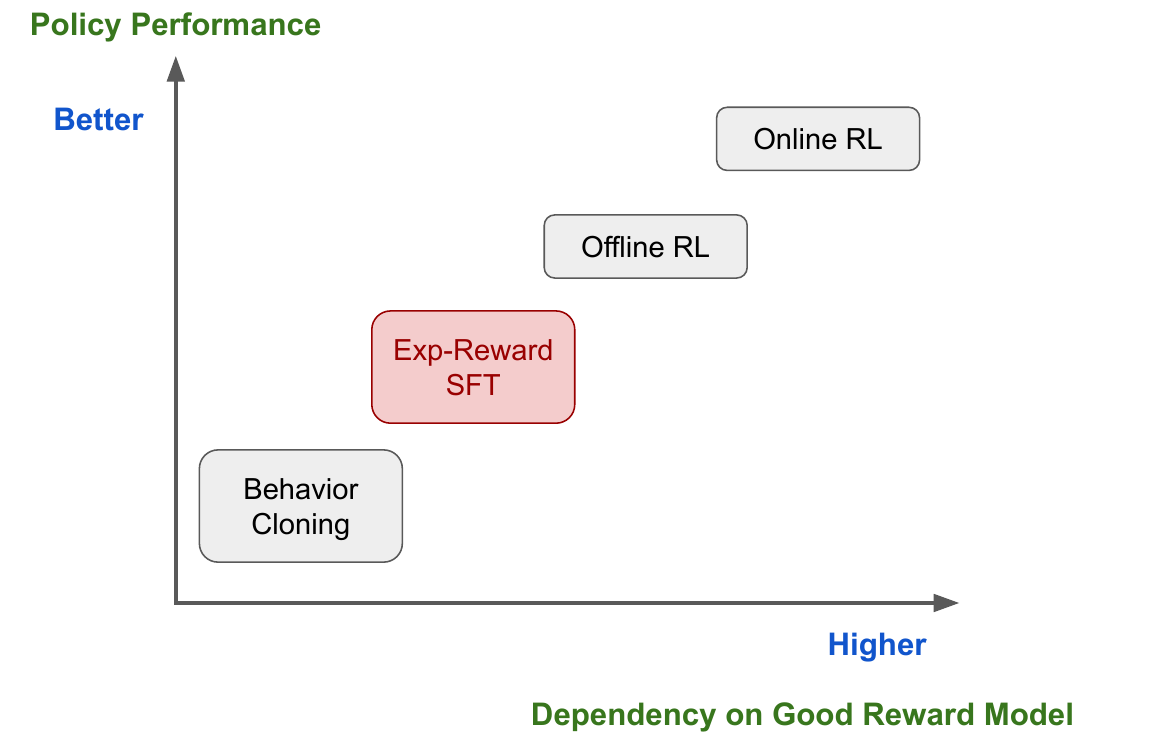}}
    \caption{
      Comparison of policy learning methods along two axes: best achievable policy performance and dependency on a good reward model.
    }
    \label{fig:tradeoff}
  \end{center}
  \vskip -0.2in
\end{figure}

Reinforcement Learning (RL) based post-training methods for generative models span a spectrum of reward model dependency and achievable performance (\cref{fig:tradeoff}) that typically includes Online RL, Offline RL and Behavior Cloning. Our method sits the space between Offline RL and Behavior Cloning, where it leverages user preferences but in a limited manner. We review the related approaches and position our method within this landscape. 

Standard post-training methods for generative models include RLHF \cite{ouyang2022training, stiennon2020learning}, DPO \cite{rafailov2023direct}, and GRPO \cite{shao2024deepseekmath}. RLHF optimizes a policy via PPO against a learned reward model. GRPO follows the same paradigm but replaces the learned value function baseline with a group average reward. Both could in principle be applied to generative recommender post-training, but the learned reward model is the central bottleneck: in recommendation settings, reward models generalize poorly across the full item catalog. Any method that queries a learned reward model at training time inherits this generalization failure, regardless of the policy optimization strategy used on top. DPO sidesteps reward model optimization but requires binary preference pairs. Recommendation feedback such as ratings, watch time, and reviews are scalar in nature, and obtaining pairwise comparisons for unobserved items still requires a learned reward model, reintroducing the generalization problem. Offline RL methods like CQL \cite{kumar2020conservative} offer principled alternatives by learning conservative Q-functions, but these must also generalize across the massive item space and face the same challenges. Doubly robust methods \cite{dudik2011doubly} and counterfactual risk minimization \cite{swaminathan2015batch} require inverse propensity scores, which are intractable to compute for complex production pipelines.

Our method is closely related to weighted behavior cloning algorithms that use monotonic functions of rewards as weights \cite{peters2007using, wang2018exponentially, liang2022local, mukherjeeoffline, kostrikov2021offline, wang2020critic}. Like our approach, these methods operate on offline data and do not require propensity scores. Most notably, \cite{peters2007using} introduces the same softmax reward weighting scheme, though their analysis focuses on training stability. We focus on a different perspective: statistical guarantees under noisy rewards and the algorithm's suitability for post-training generative recommenders. MARWIL \cite{wang2018exponentially}, IQL \cite{kostrikov2021offline}, CRR \cite{wang2020critic}, and LPI \cite{liang2022local} weight log probabilities by $\exp(\text{A}/\lambda)$ where $A$ is the advantage function. In the bandit setting, weighting by $\exp(r/\lambda)$ optimizes the same objective, as the state-dependent value function in the advantage cancels within the softmax \cite{liang2022local}. This equivalence makes explicit advantage estimation unnecessary in our case, a significant departure from methods like IQL that require learning separate $Q$ and $V$ functions via expectile regression.

\citet{mukherjeeoffline} considers a simpler variant that weights linearly by rewards. This similarly avoids reward model dependence, but the resulting policy is proportional to the product of reward and logging frequency, conflating item quality with popularity. Additionally, linear reward weighting is not robust to noisy observed rewards. In contrast, weighting by $\exp(r/\lambda)$ introduces $\lambda$ as an intrinsic regularizer that both controls sensitivity to reward noise and decouples item quality from logging frequency: as $\lambda$ decreases, the policy concentrates on high-reward items regardless of their original popularity. \citet{mukherjeeoffline} also relies on standardization to control variance; in our method, standardization is unnecessary as additive shifts cancel in the softmax and tuning $\lambda$ achieves the same effect as scaling. Methods that weight linearly with rewards or advantages face an additional challenge: negative weights lead to negative log likelihoods. In recommendation settings this is particularly harmful, as decreasing probability on observed items reallocates mass to unobserved items, which typically have lower true preference due to selection bias \cite{steck2010training, schnabel2016recommendations}, thereby leading to worse performance.

Reward-weighted methods have also gained traction for LLM post-training, notably RLOO \cite{ahmadian2024back} and ALoL \cite{baheti2023leftover}. However, ALoL requires propensity scores and a learned value function for computing advantages. RLOO generates multiple outputs and scores each with a reward model, weighting by the reward relative to the average across generations. Since we cannot generate new recommendations and observe their rewards, RLOO requires reward model predictions, reintroducing generalization issues. It also inherits the negative log-likelihood problem since rewards relative to the average can be negative.

\section{Background}

\subsection{Problem Setting}
We model recommendation as a contextual bandit with environment $\mathcal{M} = (\mathcal{S}, \mathcal{A}, r, d_0),$ where $\mathcal{S}$ is the context space, $\mathcal{A}$ the action space, $r : \mathcal{S} \times \mathcal{A} \rightarrow \mathbb{R}$ the reward function, and $d_0$ the context distribution. In the generative recommender setting, $s \in \mathcal{S}$ represents a user’s interaction history and $a \in \mathcal{A}$ the next item to recommend. We assume access to an offline dataset $\mathcal{D} = \{(s_i, a_i, r_i)\}_{i=1}^N$ collected under an unknown logging policy or mixture of policies. Since our objective is to maximize immediate reward, we model each recommendation independently as a contextual bandit. For a stochastic policy $\pi(a \mid s)$, the value and advantage functions are $V^{\pi}(s) = \mathbb{E}_{a \sim \pi(\cdot \mid s)} [ r(s, a)]$ and $A^{\pi}(s, a) = r(s, a) - V^{\pi}(s)$. 

\subsection{Exponential Reward-Weighted Policy Optimization}\label{sec:exp_rw}

Our policy optimization objective is the same as in \citet{nair2020awac}. The original arguments in \citet{nair2020awac} can be extended to obtain our algorithm, we re-state it here for completeness. 

At each iteration $k$, we seek a new policy $\pi_{k+1}$ that maximizes expected advantage under the behavior policy's state visitation distribution, while remaining close to the data-generating distribution $\pi_\beta$:
\begin{align}\label{eq:constrained_opt}
&\pi_{k+1} = \arg\max_{\pi} \E_{s \sim d^{\pi_\beta}, a \sim \pi(\cdot|s)} \left[ A^{\pi_k}(s, a) \right] \quad \\
&\text{s.t.} \quad \E_{s \sim d^{\pi_\beta}} \left[ D_{\mathrm{KL}}\left( \pi(\cdot|s) \| \pi_\beta(\cdot|s) \right) \right] \leq \epsilon,
\end{align}
where $\pi$ must be a valid policy satisfying $\sum_a \pi(a|s) = 1$ and $\pi(a|s) \geq 0$ for all $s, a$.

The Lagrangian of this constrained optimization problem is:
\begin{align}
\mathcal{L}(\pi, \lambda) &= \E_{s \sim d^{\pi_\beta}} \left[ \E_{a \sim \pi(\cdot|s)} \left[ A^{\pi_k}(s, a) \right] \right. \nonumber \\
&\quad \left. - \lambda D_{\mathrm{KL}}\left( \pi(\cdot|s) \| \pi_\beta(\cdot|s) \right) \right] + \lambda \epsilon.
\end{align}

Taking the functional derivative with respect to $\pi$ and setting it to zero yields the closed-form solution:
\begin{equation}
\pi^*(a|s) = \frac{1}{Z(s)} \pi_\beta(a|s) \exp\left( \frac{1}{\lambda} A^{\pi_k}(s, a) \right),
\label{eq:optimal_nonparametric}
\end{equation}
where $Z(s) = \sum_{a'} \pi_\beta(a'|s) \exp\left( \frac{1}{\lambda} A^{\pi_k}(s, a') \right)$ is the state-dependent partition function. This is the algoritm presented in the paper \citet{nair2020awac}. But we can simplify this further for our setting.

In a contextual bandit formulation, the advantage decomposes as $A^{\pi}(s, a) = r(s, a) - V^{\pi}(s)$, where the baseline $V^{\pi}(s)$ depends only on the state. Substituting into Equation~\eqref{eq:optimal_nonparametric}:
\begin{align}
\pi^*(a|s) &= \frac{1}{Z(s)} \pi_\beta(a|s) \exp\left( \frac{r(s,a) - V^{\pi_k}(s)}{\lambda} \right) \\
&= \frac{1}{Z(s)} \pi_\beta(a|s) \exp\left( \frac{r(s,a)}{\lambda} \right) \exp\left( \frac{-V^{\pi_k}(s)}{\lambda} \right).
\end{align}
Since $\exp\left( \frac{-V^{\pi_k}(s)}{\lambda} \right)$ is constant with respect to actions, it cancels out along with the partition function $Z(s)$:
\begin{equation}
\pi^*(a|s) = \frac{\pi_\beta(a|s) \exp\left( \frac{r(s,a)}{\lambda} \right)}{\sum_{a'} \pi_\beta(a'|s) \exp\left( \frac{r(s,a')}{\lambda} \right)}.
\label{eq:reward_weighted_policy}
\end{equation}

\begin{observation}[Baseline and Scale Invariance]
\label{obs:invariance}
In the contextual bandit setting, the optimal policy in Equation~\eqref{eq:reward_weighted_policy} exhibits two invariances:
\begin{enumerate}
    \item \textbf{Baseline invariance:} For any state-dependent function $b(s)$, adding it to the rewards leaves the optimal policy unchanged:
    \begin{align*}
    \pi^*(a|s) &\propto \pi_\beta(a|s) \exp\left( \frac{r(s,a) + b(s)}{\lambda} \right) \\
    &\propto \pi_\beta(a|s) \exp\left( \frac{r(s,a)}{\lambda} \right).
    \end{align*}
    Since $\exp(b(s)/\lambda)$ is constant with respect to $a$, it cancels in the normalization. Consequently, the policy depends only on rewards $r(s,a)$, not advantages $A^\pi(s,a) = r(s,a) - V^\pi(s)$. This eliminates the need to estimate a value function, removing a source of approximation error.
    \item \textbf{Scale invariance:} For any constant $c > 0$, scaling rewards by $c$ is equivalent to rescaling temperature: $\exp(c \cdot r / \lambda) = \exp(r / (\lambda/c))$. Thus, $\lambda$ subsumes reward scaling, and we may use unnormalized rewards without loss of generality.
\end{enumerate}
\end{observation}

Together, these properties yield a simple algorithm with a single interpretable hyperparameter $\lambda$ that controls regularization strength. As in \citet{nair2020awac}, in order to obtain a parametric policy $\pi_\theta$ (where $\theta$ denotes the policy parameters), we project the non-parametric solution onto the policy class by minimizing the KL divergence:
\begin{equation}
\theta_{k+1} = \arg\min_{\theta} \E_{s \sim d^{\pi_\beta}} \left[ D_{\mathrm{KL}}\left( \pi^*(\cdot|s) \| \pi_\theta(\cdot|s) \right) \right].
\end{equation}

Expanding the KL divergence and dropping terms independent of $\theta$:
\begin{align}
\theta_{k+1} &= \arg\max_{\theta} \E_{s \sim d^{\pi_\beta}} \left[ \E_{a \sim \pi^*(\cdot|s)} \left[ \log \pi_\theta(a|s) \right] \right] \\
&= \arg\max_{\theta} \E_{s \sim d^{\pi_\beta}, a \sim \pi_\beta(\cdot|s)} \left[ \frac{\pi^*(a|s)}{\pi_\beta(a|s)} \log \pi_\theta(a|s) \right] \\
&= \arg\max_{\theta} \E_{(s,a) \sim \mcD} \left[ \exp\left( \frac{r(s, a)}{\lambda} \right) \log \pi_\theta(a|s) \right],
\label{eq:rw_objective}
\end{align}
where we used importance sampling to convert the expectation over $\pi^*$ to an expectation over $\pi_\beta$, and dropped the partition function $Z(s)$.

Equation~\eqref{eq:rw_objective} corresponds to a weighted maximum likelihood objective where actions from the offline dataset are weighted by exponentiated scaled rewards, yielding Algorithm~\ref{alg:exp-rw-sft}. The algorithm is simple to implement via standard SFT APIs and requires neither a reward model, importance sampling, nor knowledge of the logging policy $\pi_\beta$. The only hyperparameter is $\lambda$, which also implicitly controls regularization strength. In the next section, we analyze the robustness of this method to noise in $r$, providing both a high-probability policy improvement guarantee and an instance-dependent top-$k$ recoverability guarantee under noisy rewards.

\begin{algorithm}[tb]
\caption{Exponential Reward Weighted SFT}
\label{alg:exp-rw-sft}
\begin{algorithmic}
\STATE {\bfseries Input:} Offline dataset $\mathcal{D} = \{(s_i, a_i, r_i)\}_{i=1}^N$, initial policy $\pi_{\theta}$, temperature $\lambda > 0$, number of epochs $T$
\FOR{$t = 1, \ldots, T$}
    \FOR{each mini-batch $B \subset \mathcal{D}$}
        \STATE Compute loss: 
        \STATE $\mathcal{L}(\theta) \gets -\frac{1}{|B|} \displaystyle\sum_{(s, a, r) \in B} \exp\left(\frac{r}{\lambda}\right) \cdot \log \pi_\theta(a \mid s)$
        \STATE Update: $\theta \gets \theta - \eta \nabla_\theta \mathcal{L}(\theta)$
    \ENDFOR
\ENDFOR
\STATE {\bfseries Output:} Fine-tuned policy $\pi_{\theta}$
\end{algorithmic}
\end{algorithm}

\section{Theoretical Analysis}
We begin our analysis of Algorithm~\ref{alg:exp-rw-sft} with a warm-up exercise that assumes that the offline dataset contains the true ground truth rewards $r^*$. The monotonic policy improvement proof from \cite{wang2018exponentially} can be extended for this case to give the following policy improvement result (full proof in Appendix \ref{subsec:monotonic-proof}): 

\begin{proposition}[Monotonic Policy Improvement]
\label{prop:monotonic}
Let $\pi^*_{\lambda}(a|s) \propto \pi_\beta(a|s) \exp\left(\frac{r^*(s,a)}{\lambda}\right)$. Then in the contextual bandit setting:
\begin{align*}
V^{\pi^*_{\lambda}}(s) \geq V^{\pi_\beta}(s) \quad \text{for all } s \in \mathcal{S}.
\end{align*}
\end{proposition}
This simple proof shows that by performing SFT with exponentially-weighted rewards scaled by a temperature, we obtain a policy that is monotonically better than the logging policy. However, in practice, user feedback signals are inherently stochastic: a user's rating for the same item may vary across sessions due to mood, context, or attention; click behavior is noisy and influenced by position bias; and implicit signals like dwell time fluctuate based on external factors unrelated to true preference. Next, we aim to understand policy improvement from this noisy reward perspective. We being by making the following simplifying assumptions on the noise:

\begin{assumption}[Sub-Gaussian Reward Noise]
\label{ass:noise}
Observed rewards satisfy $\hat{r}(s,a) = r^*(s,a) + \xi(s,a)$, where $r^*$ denotes the true expected reward for state-action pair $(s,a)$ and $\{\xi(s,a)\}_{a \in \mathcal{A}}$ are independent, zero-mean, $\sigma$-sub-Gaussian random variables.
\end{assumption}

The sub-gaussian condition is standard and encompasses most reward distributions in recommender systems: bounded signals such as discrete ratings or binary clicks are sub-Gaussian by Hoeffding's lemma, as are continuous rewards with Gaussian noise. The zero-mean and independence conditions are standard in stochastic bandit analysis \cite{auer2002finite}. Given this noise structure, we derive the following guarantee (see Appendix \ref{subsec:noisy-improvement-proof} for the full proof):

\begin{theorem}[Policy Improvement Under Noisy Rewards]
\label{thm:noisy_improvement}
Let $\pi_\beta$ be the behavior policy and $\pi_\lambda$ be the policy obtained by exponential reward-weighted SFT with temperature $\lambda > 0$ using noisy rewards $\hat{r} = r^* + \xi$, where $\xi$ is $\sigma$-sub-Gaussian (Assumption~\ref{ass:noise}). With probability at least $1 - \delta$:
\begin{equation}
V^{\pi_\lambda}(s) \geq V^{\pi_\beta}(s) - 2\sigma\sqrt{2\log(2|\mathcal{A}|/\delta)},
\end{equation}
where $|\mathcal{A}|$ is the number of actions.
\end{theorem}

The gap scales as $O(\sigma\sqrt{\log|\mathcal{A}|})$, so the logarithmic dependence on catalog size means the bound remains informative even for large item catalogs. In the noiseless limit ($\sigma \to 0$), the bound recovers the exact monotonic improvement guarantee of Proposition~\ref{prop:monotonic}. Note that the bound holds pointwise for each state $s$; a uniform guarantee over all states would require an additional union bound over $|\mathcal{S}|$, yielding an $O(\sigma\sqrt{\log(|\mathcal{S}||\mathcal{A}|)})$ dependence. Note that this bound is independent of the regularizer/temperature $\lambda$, if the observed rewards are bounded, we are able to work out a temperature-dependent bound as follows (full proof in \ref{subsec:temp-bound-proof}):

\begin{theorem}[Temperature-Dependent Bound]
\label{thm:temp_bound}
Under the conditions of Theorem~\ref{thm:noisy_improvement}, if rewards are bounded $r^*(s,a) \in [0, R_{\max}]$, then with probability at least $1-\delta$:
\begin{equation}
    V^{\pi_\lambda}(s) \geq V^{\pi_\beta}(s) - R_{\max}\left(e^{2\epsilon/\lambda} - 1\right),
\end{equation}
where $\epsilon = \sigma\sqrt{2\log(2|\mathcal{A}|/\delta)}$. When $\lambda \geq 2\epsilon$, this simplifies to:
\begin{equation}
    V^{\pi_\lambda}(s) \geq V^{\pi_\beta}(s) - \frac{4R_{\max}\epsilon}{\lambda}.
\end{equation}
\end{theorem}

Theorem~\ref{thm:temp_bound} reveals an explicit robustness-improvement tradeoff controlled by $\lambda$. Large $\lambda$ suppresses the noise cost ($e^{2\epsilon/\lambda} - 1 \to 0$) but drives $\pi_\lambda \to \pi_\beta$, yielding no improvement. Small $\lambda$ enables aggressive re-ranking but amplifies sensitivity to reward noise. In the regime $\lambda \geq 2\epsilon$, the linearized bound $4R_{\max}\epsilon/\lambda$ shows that the noise cost decays as $O(1/\lambda)$. Given noise level $\sigma$ and desired degradation tolerance $\tau > 0$, a sufficient condition for the value degradation to remain within $\tau$ is:
\begin{equation}
    \lambda \geq \frac{2\epsilon}{\log(1 + \tau/R_{\max})},
\end{equation}
which follows from requiring $R_{\max}(e^{2\epsilon/\lambda} - 1) \leq \tau$. While this lower bound is conservative due to worst-case analysis over the full action space, it reveals the qualitative structure of the tradeoff: the minimum safe temperature scales linearly with the noise level $\sigma$ and inversely with the log-tolerance $\log(1 + \tau/R_{\max})$. In practice, we recommend tuning $\lambda$ via validation, as the theoretical lower bound can be significantly loose.

\section{Extension to the MDP Setting}
\label{sec:mdp}

The objective in Equation~\eqref{eq:constrained_opt} extends directly to the MDP setting:
\begin{align}
    &\mathbb{E}_{s \sim d^{\pi_\beta}, a \sim \pi}\!\left[A^{\pi_k}(s,a)\right] - \lambda D_{\mathrm{KL}} \nonumber \\
    &= \mathbb{E}_{s \sim d^{\pi_\beta}, a \sim \pi}\!\left[r(s,a)\right] - 
    \underbrace{\mathbb{E}_{s \sim d^{\pi_\beta}}\!\left[V^{\pi_k}(s)\right]}_{\text{constant w.r.t.}\ \pi} 
    - \lambda D_{\mathrm{KL}}.
\end{align}
Since $d^{\pi_\beta}$ is fixed and $V^{\pi_k}(s)$ does not depend on $a$, the value term is constant with respect to $\pi$ and drops out. Since $d^{\pi_\beta}(s) = \frac{1}{T}\sum_{t=1}^T P(s_t = s \mid \pi_\beta)$ marginalizes over timesteps, the remaining objective satisfies:
\begin{equation}
    \mathbb{E}_{s \sim d^{\pi_\beta}, a \sim \pi}\!\left[r(s,a)\right] = 
    \frac{1}{T}\mathbb{E}_\tau\!\left[R(\tau)\right],
\end{equation}
where $R(\tau) = \sum_{t=1}^T r(s_t, a_t)$ is the trajectory-level rewards. Since $\frac{1}{T}$ is constant with respect to $\pi$, the objective is equivalent to:
\begin{equation}
    \max_\pi \; \mathbb{E}_\tau\!\left[R(\tau)\right] - \lambda D_{\mathrm{KL}}.
\end{equation}

The contextual bandit setting is recovered as the special case $T=1$ where $R(\tau) = r(s,a)$. Since the objectives are identical, the same Lagrangian derivation applies with $R(\tau)$ in place of $r(s,a)$, yielding the trajectory-level optimal policy:
\begin{equation}
    \pi^*(\tau) \propto \pi_\beta(\tau)\exp\!\left(\frac{R(\tau)}{\lambda}\right),
\end{equation}
and the corresponding training objective:
\begin{equation}
    \mathcal{L}(\theta) = -\frac{1}{|\mathcal{D}|}\sum_{\tau \in \mathcal{D}}
    \exp\!\left(\frac{R(\tau)}{\lambda}\right)\sum_{t=1}^T
    \log\pi_\theta(a_t \mid s_t).
    \label{eq:mdp_objective}
\end{equation}

The theoretical guarantees of Theorems~\ref{thm:noisy_improvement} 
and~\ref{thm:temp_bound} hold without modification, requiring only that the observed return $\hat{R}(\tau) = R^*(\tau) + \xi(\tau)$ satisfies Assumption~\ref{ass:noise}. This holds under either of the following conditions: (i) if per-step noise terms are independent, sub-Gaussianity of each implies sub-Gaussianity of their sum by standard closure properties; (ii) if per-step noise terms are dependent, as they naturally are along a Markov trajectory, we require only that the trajectory-level aggregate noise $\xi(\tau) = \sum_{t=1}^T \xi_t$ is itself sub-Gaussian, which we take as a direct assumption on the noise process. Condition (i) is a special case of (ii).

One practical consideration is reward variance. Since $R(\tau)$ is a sum of up to $T$ per-step rewards, its variance grows with trajectory length. The theoretically principled way to control the influence of high-variance returns is through $\lambda$ itself. Theorem~\ref{thm:temp_bound} explicitly characterizes this tradeoff, showing the cost of variance reduction in terms of policy improvement. In practice, standardizing $\tilde{R}(\tau) = (R(\tau) - \mu)/\sigma$ before exponential weighting serves as a useful preprocessing step that stabilizes the numerical range of weights and makes $\lambda$ more interpretable across datasets with different return scales, corresponding to a reparametrization $\lambda \to \lambda\sigma$ by Observation~\ref{obs:invariance}.

\section{Experiments}
\begin{table*}[!htbp]
\centering
\caption{Dataset Statistics}
\label{tab:dataset_stats}
\begin{tabular}{lccccc}
\toprule
Dataset & Users & Items & Avg. Interactions/User & Test Cases ($r \geq 4.5$) \\
\midrule
ML-1M & 6,040 & 3,706 & 165.60 & 1,530 \\
ML-20M & 138,493 & 26,744 & 144.41 & 40,667 \\
Amazon Books & 811,227 & 695,762 & 14.47 & 424,271 \\
Netflix & O(Millions) & O(Thousands) & O(Tens) & O(Thousands) \\
\bottomrule
\end{tabular}
\end{table*}

\textbf{Dataset:} We use three open-source recommendation datasets containing sequential user-item interactions: MovieLens 1M (ML-1M), MovieLens 20M (ML-20M)~\cite{harper2015movielens}, and Amazon Reviews~\cite{mcauley2015image}. In addition, we also present results on a large-scale proprietary dataset from Netflix. These datasets contains data for multiple users. For each user, it contains interaction data, i.e. what items did the user interact with and the corresponding rewards the user assigned to each interaction.

Table~\ref{tab:dataset_stats} summarizes the count of users, items and average interactions per user for each of the three dataset. Because of the proprietary nature of the Netflix data, we only show the order of magnitude for the data instead of the actual count.
The open source datasets vary significantly in scale and sparsity, with ML-1M being the smallest and densest, ML-20M providing a larger movie domain, and Amazon Books offering the most diverse item catalog with sparser user interactions, while the proprietary Netflix data is the largest of them all and helps showcase the performance of our data on realistic large-scale real-world data. Additional information about dataset rating distribution is shown in \ref{fig:rating_dist} in the Appendix.

\textbf{Model:} We conduct experiments using the HSTU~\cite{zhai2024actions} base model pretrained on each of the datasets. This will serve as the baseline for our method and represents the Behavior Cloning (BC) algorithm that blindly mimics each user. Next we co-train a reward model as a shallow reward head on top of the generative recommender. It predicts the reward for the most recently selected title based on a user's interaction history. 

\textbf{Algorithms:} We benchmark our method against four algorithms spanning behavior cloning, weighted SFT and standard RLHF:
\begin{itemize}
    \item \textbf{BC (Behavior Cloning)}: The base HSTU architecture \cite{zhai2024actions} trained to predict a user's next interaction given their history, without optimizing for any reward signal.
    \item \textbf{Reward-SFT (Reward-weighted Supervised Fine-Tuning}: The model is post-trained using reward-weighted log probabilities as described in \cite{mukherjeeoffline}.
    \item \textbf{DPO (Direct Preference Optimization)}: This is policy optimization method that directly learns from pairwise preference data by increasing the likelihood of preferred actions over dispreferred ones \cite{rafailov2023direct}. Since binary comparison data is unavailable in recommendation systems, we implement an online variant of DPO where binary preferences are obtained by generating two responses and comparing them via reward model scores, similar to the online DPO in \cite{deng2025onerec}.
    \item \textbf{PPO (Proximal Policy Optimization)}: A standard RLHF approach \cite{stiennon2020learning, ouyang2022training} that updates the policy using clipped probability ratios and estimated advantages computed from the learned reward model \cite{schulman2017proximal}.
    \item \textbf{Exp-RSFT (Exponential Reward-weighted Supervised Fine-Tuning)}: Our algorithm, described in \ref{alg:exp-rw-sft}, post-trains the model by weighting log-probabilities with the exponential of the reward scaled by a temperature parameter.
\end{itemize}

We run the algorithms for 30, 12, and 10 epochs for the ML-1M, ML-20M, and Amazon Books datasets.

\textbf{Evaluation:} We evaluate on the standard test splits of ML-1M, ML-20M, and Amazon Books, but restrict our metrics to test cases where the held-out target item received a rating $\geq 4.5$. This design choice aligns evaluation with our training objective: we aim to measure each method's ability to surface items that users would rate highly, rather than simply predicting any future interaction regardless of user satisfaction. Table~\ref{tab:dataset_stats} reports the number of test cases meeting this criterion for each dataset.   

On this filtered test set, we report five standard recommendation metrics. Let $r_i$ denote the rank assigned to the ground truth item for test case $i$ among $N$ total cases. \textit{Hit Rate} (HR@$K$) measures the proportion of cases where the ground truth item appears in the top-$K$ recommendations. \textit{Normalized Discounted Cumulative Gain} (NDCG@$K$) additionally accounts for ranking position, assigning higher scores to ground truth items ranked near the top. \textit{Mean Reciprocal Rank} (MRR) averages the inverse rank across all test cases. Formally:
\begin{align}
    \text{HR@}K &= \frac{1}{N} \sum_{i=1}^{N} \mathbbm{1}(r_i \leq K), \\
    \text{NDCG@}K &= \frac{1}{N} \sum_{i=1}^{N} \frac{\mathbbm{1}(r_i \leq K)}{\log_2(r_i + 1)}, \\
    \text{MRR} &= \frac{1}{N} \sum_{i=1}^{N} \frac{1}{r_i}.
\end{align}
We report HR and NDCG at cutoffs $K \in \{10, 50\}$.

\begin{figure*}[t]
  \vskip 0.2in
  \begin{center}    \centerline{\includegraphics[width=\columnwidth*2]{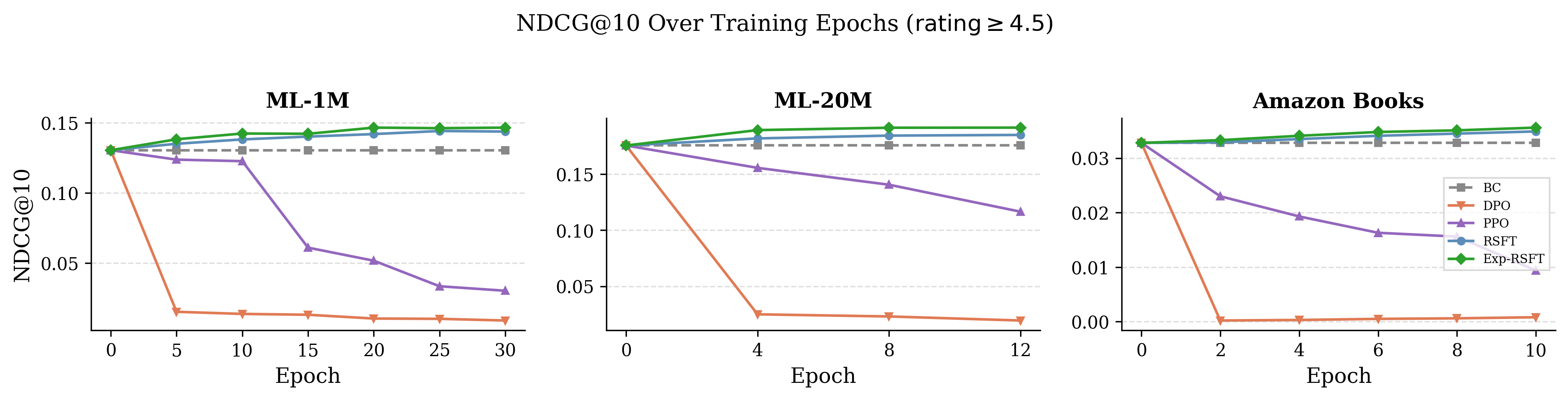}}
    \caption{
      Figure showing the dramatic collapse of PPO and DPO due to reward-hacking.
    }
    \label{fig:main_metrics}
  \end{center}
  \vskip -0.2in
\end{figure*}

\begin{figure*}[t]
  \vskip 0.2in
  \begin{center}    \centerline{\includegraphics[width=\columnwidth*2]{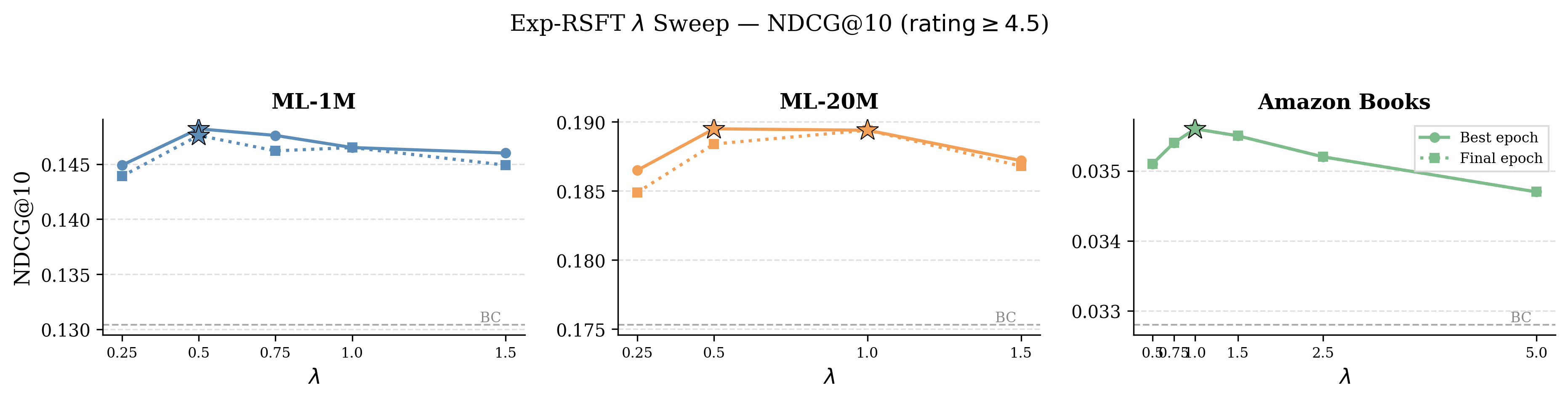}}
    \caption{
      NDCG@10 for different values of $\lambda$ for all three datasets.
    }
    \label{fig:lambda_main}
  \end{center}
  \vskip -0.2in
\end{figure*}

\textbf{Results:} The Tables~\ref{tab:ml1m_results}, \ref{tab:ml20m_results}, \ref{tab:amzn_results}, and \ref{tab:netflix_results} summarize the results for the datasets ML-1M, ML-20M, Amazon Books, and Netflix respectively. We see that our algorithm (Exp-RSFT) consistently outperforms the other benchmark algorithms across the four datasets. We also observe this trend across the training epochs (see Appendix \ref{app:metric_evolution} for evolution of metrics across training epochs for the open source datasets). 

\textbf{Evidence for Reward Hacking:} Tables~\ref{tab:ml1m_results}, \ref{tab:ml20m_results}, \ref{tab:amzn_results}, and \ref{tab:netflix_results} show that both PPO and DPO experience catastrophic collapse across all datasets. We hypothesize that this is due to reward hacking and test this hypothesis using two approaches:

\textit{Benchmarking reward model training metrics:} To assess the quality of the learned reward model, we compare its Mean Squared Error (MSE) and Mean Absolute Error (MAE) against three simple baselines:
\begin{itemize}
    \item \textbf{User Mean:} Predicts future rewards as the historical average reward for each user.
    \item \textbf{Item Mean:} Predicts the reward for an item as its historical average across all users.
    \item \textbf{Global Mean:} Predicts all rewards as the global mean rating across all users and items.
\end{itemize}

We report these metrics in Table~\ref{tab:rm_baselines} for all four datasets. We observe that, except for MAE on ML-20M and MSE on Netflix, the reward model does not achieve the lowest error. Even in the two cases where it does, its performance is comparable to the simple item-mean baseline. This suggests that the reward model offers limited improvement over naive predictors. This result is intuitive, since a user's historical interactions typically cover only a small subset of the catalog, making it difficult to accurately predict their responses to unexplored items.

\textit{Reward model-as-judge evaluation:} Given the limited generalization capacity of reward models, we investigate whether each algorithm over-optimizes for this imperfect signal, a phenomenon known in the LLM literature as reward hacking~\cite{pan2022effects, skalse2022defining, lambert2023alignment, rafailov2024scaling}. We compute the average reward model score across 10 generations from each of 10,000 evaluation contexts, reported as ``Avg Reward'' in Tables~\ref{tab:ml1m_results}, \ref{tab:ml20m_results}, \ref{tab:amzn_results}, and \ref{tab:netflix_results}. We find that this metric does not correlate with true recommendation quality. Notably, PPO and DPO achieve the highest reward model scores across all four datasets, suggesting that these algorithms over-optimize for a reward signal that fails to capture genuine user preferences.

\textbf{Robustness–Improvement Tradeoff with $\lambda$:} Theorem~\ref{thm:temp_bound} predicts a fundamental tension in the choice of temperature $\lambda$: small $\lambda$ aggressively re-weights toward high-reward actions but amplifies sensitivity to reward noise, while large $\lambda$ suppresses the noise cost but drives the learned policy toward the behavior policy, yielding diminishing improvement. We validate this tradeoff empirically by sweeping a range of $\lambda$ values across all three datasets. Figure~\ref{fig:lambda_main} illustrates the trend for NDCG@10; other metrics exhibit similar behavior and are reported in Appendix~\ref{app:lambda-sweep}. All five evaluation metrics display a clear inverted-U shape as a function of $\lambda$ across all three datasets, with performance peaking near $\lambda \approx 0.5$--$1.0$ and declining on either side. This pattern holds whether we select the best or final training epoch, ruling out early stopping as a confound. As expected, large data reduces policy improvement as it theoretically would recover the behavior policy. But more surprisingly, and as predicted by Theorem~\ref{thm:temp_bound}, values of $\lambda$ smaller than the optimum also lead to degraded performance. These results confirm that moderate reward weighting strikes the optimal balance between exploiting the reward signal and regularizing against its noise.

\FloatBarrier

\begin{table*}[!htbp]
  \centering
  \caption{Reward model prediction performance compared to naive baselines. The trained reward model fails to outperform the simple Item Mean baseline on most metrics. Best results per row are \textbf{bolded}.}
  \label{tab:rm_baselines}
  \begin{tabular}{llcccc}
    \toprule
    Dataset & Metric & User Mean & Item Mean & Global Mean & Trained RM \\
    \midrule
    \multirow{2}{*}{ML-1M}
      & MSE & 1.0784 & \textbf{0.9179} & 1.2394 & 1.2166 \\
      & MAE & 0.8232 & \textbf{0.7619} & 0.9268 & 0.8680 \\
    \midrule
    \multirow{2}{*}{ML-20M}
      & MSE & 0.9550 & \textbf{0.8898} & 1.0939 & 0.8963 \\
      & MAE & 0.7588 & 0.7320 & 0.8423 & \textbf{0.7307} \\
    \midrule
    \multirow{2}{*}{Amazon Books}
      & MSE & 1.1578 & \textbf{0.7290} & 1.1868 & 1.1067 \\
      & MAE & 0.7411 & \textbf{0.5651} & 0.8636 & 0.8140 \\
    \midrule
    \multirow{2}{*}{Netflix}
      & MSE & \textbf{2.412} & 2.568 & 2.568 & 2.486 \\
      & MAE & 2.112 & 2.442 & 2.451 & \textbf{1.916} \\
    \bottomrule
  \end{tabular}
\end{table*}

\begin{table*}[!htbp]
  \centering
  \caption{Recommendation performance on ML-1M (rating $\geq 4.5$). Best results per metric are \textbf{bolded}. Percentages indicate relative change from BC baseline.}
  \label{tab:ml1m_results}
  \begin{tabular}{lcccccc}
    \toprule
    Method & NDCG@10 & NDCG@50 & HR@10 & HR@50 & MRR & Avg Reward \\
    \midrule
    BC & 0.1304 & 0.1866 & 0.2405 & 0.4967 & 0.1116 & 3.8672 \\
    DPO & 0.0091 (-93.0\%) & 0.0231 (-87.6\%) & 0.0163 (-93.2\%) & 0.0824 (-83.4\%) & 0.0122 (-89.1\%) & 3.3737 (-12.8\%) \\
    PPO & 0.0303 (-76.8\%) & 0.0617 (-66.9\%) & 0.0673 (-72.0\%) & 0.2124 (-57.2\%) & 0.0285 (-74.5\%) & \textbf{4.4775} (+15.8\%) \\
    RSFT & 0.1437 (+10.2\%) & 0.2003 (+7.3\%) & 0.2608 (+8.4\%) & 0.5163 (+3.9\%) & 0.1230 (+10.2\%) & 3.6558 (-5.5\%) \\
    Exp-RSFT & \textbf{0.1465} (+12.3\%) & \textbf{0.2029} (+8.7\%) & \textbf{0.2712} (+12.8\%) & \textbf{0.5261} (+5.9\%) & \textbf{0.1235} (+10.7\%) & 3.6933 (-4.5\%) \\
    \bottomrule
  \end{tabular}
\end{table*}

\begin{table*}[!htbp]
  \centering
  \caption{Recommendation performance on ML-20M (rating $\geq 4.5$). Best results per metric are \textbf{bolded}. Percentages indicate relative change from BC baseline.}
  \label{tab:ml20m_results}
  \begin{tabular}{lcccccc}
    \toprule
    Method & NDCG@10 & NDCG@50 & HR@10 & HR@50 & MRR & Avg Reward \\
    \midrule
    BC & 0.1753 & 0.2313 & 0.3035 & 0.5569 & 0.1510 & 3.7388 \\
    DPO & 0.0196 (-88.8\%) & 0.0342 (-85.2\%) & 0.0372 (-87.7\%) & 0.1052 (-81.1\%) & 0.0195 (-87.1\%) & \textbf{4.4008} (+17.7\%) \\
    PPO & 0.1164 (-33.6\%) & 0.1688 (-27.0\%) & 0.2141 (-29.5\%) & 0.4536 (-18.6\%) & 0.1008 (-33.2\%) & 3.7547 (+0.4\%) \\
    RSFT & 0.1847 (+5.4\%) & 0.2403 (+3.9\%) & 0.3161 (+4.2\%) & 0.5677 (+1.9\%) & 0.1593 (+5.5\%) & 3.7904 (+1.4\%) \\
    Exp-RSFT & \textbf{0.1912} (+9.1\%) & \textbf{0.2462} (+6.4\%) & \textbf{0.3238} (+6.7\%) & \textbf{0.5726} (+2.8\%) & \textbf{0.1651} (+9.3\%) & 3.8611 (+3.3\%) \\
    \bottomrule
  \end{tabular}
\end{table*}

\begin{table*}[!htbp]
  \centering
  \caption{Recommendation performance on Amazon Books (rating $\geq 4.5$). Best results per metric are \textbf{bolded}. Percentages indicate relative change from BC baseline.}
  \label{tab:amzn_results}
  \begin{tabular}{lcccccc}
    \toprule
    Method & NDCG@10 & NDCG@50 & HR@10 & HR@50 & MRR & Avg Reward \\
    \midrule
    BC & 0.0328 & 0.0478 & 0.0589 & 0.1279 & 0.0296 & 4.2824 \\
    DPO & 0.0008 (-97.6\%) & 0.0020 (-95.8\%) & 0.0019 (-96.8\%) & 0.0075 (-94.1\%) & 0.0013 (-95.6\%) & \textbf{4.8615} (+13.5\%) \\
    PPO & 0.0094 (-71.3\%) & 0.0156 (-67.4\%) & 0.0184 (-68.8\%) & 0.0469 (-63.3\%) & 0.0088 (-70.3\%) & 4.4549 (+4.0\%) \\
    RSFT & 0.0349 (+6.4\%) & 0.0514 (+7.5\%) & 0.0630 (+7.0\%) & 0.1388 (+8.5\%) & 0.0314 (+6.1\%) & 4.3350 (+1.2\%) \\
    Exp-RSFT & \textbf{0.0356} (+8.5\%) & \textbf{0.0520} (+8.8\%) & \textbf{0.0641} (+8.8\%) & \textbf{0.1397} (+9.2\%) & \textbf{0.0319} (+7.8\%) & 4.3374 (+1.3\%) \\
    \bottomrule
  \end{tabular}
\end{table*}

\begin{table*}[!htbp]
  \centering
  \caption{Recommendation performance on Netflix data (high reward trajectories). Best results per metric are \textbf{bolded}. Percentages indicate relative change from R-SFT baseline.}
  \label{tab:netflix_results}
  \begin{tabular}{lcccccc}
    \toprule
    Algorithm & NDCG@10 & NDCG@50 & HR@10 & HR@50 & MRR & Reward \\
    \midrule
    RSFT & 0.00\% & 0.00\% & 0.00\% & 0.00\% & 0.00\% & 0.00\% \\
    Exp-RSFT & \textbf{106.20\%} & \textbf{69.71\%} & \textbf{98.07\%} & \textbf{49.58\%} & \textbf{86.31\%} & 17.27\% \\
    DPO & -99.30\% & -96.77\% & -99.14\% & -95.52\% & -95.34\% & \textbf{19.21\%} \\
    PPO & -4.88\% & -24.01\% & -14.66\% & -35.71\% & -8.60\% & 13.39\% \\
    \bottomrule
  \end{tabular}
\end{table*}

\section{Limitations}
We consider the specific setting where users interact with only a small fraction of a large catalog and provide scalar rewards for the items they engage with. This setting is commonly found in large-scale industrial recommendation systems. However, in different settings, for example one where reward models generalize well and can reasonably score most counterfactual actions, or one where feedback is in the form of binary comparisons covering most of the item catalog, we expect standard algorithms like PPO or DPO to perform better.

\section*{Impact Statement}
This paper presents work whose goal is to advance the field of Machine Learning. There are many potential societal consequences of our work, none which we feel must be specifically highlighted here.

\bibliography{main}
\bibliographystyle{icml2026}

\newpage
\appendix
\onecolumn
\section{Proofs}

\subsection{Proposition \ref{prop:monotonic} Proof}
\label{subsec:monotonic-proof}

\begin{proof}
Fix state $s$. Define $\mathcal{A}_+ = \{a : \pi^*_{\lambda}(a|s) \geq \pi_\beta(a|s)\}$ and $\mathcal{A}_- = \{a : \pi^*_{\lambda}(a|s) < \pi_\beta(a|s)\}$. Since $\pi^*_{\lambda}$ upweights actions with higher rewards, there exists threshold $q(s)$ such that $r^*(s,a) \geq q(s)$ for all $a \in \mathcal{A}_+$ and $r^*(s,a) \leq q(s)$ for all $a \in \mathcal{A}_-$. Then:
\begin{align*}
V^{\pi^*_{\lambda}}(s) - V^{\pi_\beta}(s) &= \sum_{a \in \mathcal{A}_+} (\pi^*_{\lambda}(a|s) - \pi_\beta(a|s)) r^*(s,a) + \sum_{a \in \mathcal{A}_-} (\pi^*_{\lambda}(a|s) - \pi_\beta(a|s)) r^*(s,a) \\
&\geq q(s) \sum_{a \in \mathcal{A}_+} (\pi^*_{\lambda}(a|s) - \pi_\beta(a|s)) + q(s) \sum_{a \in \mathcal{A}_-} (\pi^*_{\lambda}(a|s) - \pi_\beta(a|s)) \\
&= q(s) \left( \sum_a \pi^*_{\lambda}(a|s) - \sum_a \pi_\beta(a|s) \right) = 0.
\end{align*}
\end{proof}

\subsection{Theorem~\ref{thm:noisy_improvement} Proof} \label{subsec:noisy-improvement-proof}

\begin{proof}
\textbf{Step 1:} First we derive a high probability bound for the error $\xi$. Since each $\xi(s,a)$ is $\sigma$-sub-Gaussian, for any $t > 0$ and any action $a \in \mathcal{A}$ and state $s \in \mathcal{S}$:
\[
\Pr(|\xi(s,a)| \geq t) \leq 2\exp\left(-\frac{t^2}{2\sigma^2}\right).
\]
Applying a union bound over all $|\mathcal{A}|$ actions:
\[
\Pr\left(\exists a \in \mathcal{A} : |\xi(s,a)| \geq t\right) \leq 2|\mathcal{A}|\exp\left(-\frac{t^2}{2\sigma^2}\right).
\]
Setting the right-hand side equal to $\delta$ and solving for $t$:
\[
2|\mathcal{A}|\exp\left(-\frac{t^2}{2\sigma^2}\right) = \delta \implies t = \sigma\sqrt{2\log(2|\mathcal{A}|/\delta)}.
\]
Define $\epsilon := \sigma\sqrt{2\log(2|\mathcal{A}|/\delta)}$. Then with probability at least $1-\delta$:
\[
|\xi(s,a)| \leq \epsilon \quad \forall a \in \mathcal{A}.
\]
We condition on this event for the remainder of the proof. Additionally, we fix a state $s \in \mathcal{S}$ and assume $\pi_\beta(a|s) > 0$ for all $a \in \mathcal{A}$.

\textbf{Step 2:} Now we can proceed to  compare $V^{\pi_\lambda}(s)$ and $V^{\pi_\beta}(s)$. $\pi_\lambda(a|s) \propto \pi_\beta(a|s)\exp(\hat{r}(s,a)/\lambda)$ is the exponential tilt using noisy rewards $\hat{r} = r^* + \xi$. From Section~\ref{sec:exp_rw}, we know that $\pi_\lambda(\cdot|s)$ is the unique maximizer of
\[
F(\pi) = \sum_a \pi(a|s)\,\hat{r}(s,a) - \lambda\, D_{\mathrm{KL}}\left( \pi(\cdot|s) \| \pi_\beta(\cdot|s) \right) 
\]
In particular, $F(\pi_\lambda) \geq F(\pi_\beta)$. Since $\mathrm{KL}(\pi_\beta \| \pi_\beta) = 0$, this gives:
\[
\sum_a \pi_\lambda(a|s)\,\hat{r}(s,a) - \lambda\, D_{\mathrm{KL}}(\pi_\lambda \| \pi_\beta) \geq \sum_a \pi_\beta(a|s)\,\hat{r}(s,a).
\]
Since $\lambda > 0$ and $\mathrm{KL}(\pi_\lambda \| \pi_\beta) \geq 0$, the left-hand side is at most $\sum_a \pi_\lambda(a|s)\,\hat{r}(s,a)$, so:
\begin{equation}\label{eq:dominance}
\sum_a \pi_\lambda(a|s)\,\hat{r}(s,a) \geq \sum_a \pi_\beta(a|s)\,\hat{r}(s,a).
\end{equation}
It remains to convert this inequality on noisy rewards into an inequality on true rewards. Since $|\xi(s,a)| \leq \epsilon$ for all $a$ and $\pi_\lambda(\cdot|s)$ is a probability distribution:
\begin{equation}\label{eq:noise1}
V^{\pi_\lambda}(s) = \sum_a \pi_\lambda(a|s)\,r^*(s,a) = \sum_a \pi_\lambda(a|s)\,(\hat{r}(s,a) - \xi(s,a)) \geq \sum_a \pi_\lambda(a|s)\,\hat{r}(s,a) - \epsilon.
\end{equation}
Similarly, since $\hat{r}(s,a) = r^*(s,a) + \xi(s,a) \geq r^*(s,a) - \epsilon$ for all $a$:
\begin{equation}\label{eq:noise2}
\sum_a \pi_\beta(a|s)\,\hat{r}(s,a) \geq \sum_a \pi_\beta(a|s)\,(r^*(s,a) - \epsilon) = V^{\pi_\beta}(s) - \epsilon.
\end{equation}
Chaining \eqref{eq:noise1}, \eqref{eq:dominance}, and \eqref{eq:noise2}:
\[
V^{\pi_\lambda}(s) \geq \sum_a \pi_\lambda(a|s)\,\hat{r}(s,a) - \epsilon \geq \sum_a \pi_\beta(a|s)\,\hat{r}(s,a) - \epsilon \geq V^{\pi_\beta}(s) - 2\epsilon.
\]
Substituting $\epsilon = \sigma\sqrt{2\log(2|\mathcal{A}|/\delta)}$ completes the proof.
\end{proof}

\subsection{Theorem~\ref{thm:temp_bound} Proof}
\label{subsec:temp-bound-proof}

Under the conditions of Theorem~\ref{thm:noisy_improvement}, if rewards are bounded $r^*(s,a) \in [0, R_{\max}]$, then with probability at least $1-\delta$:
\begin{equation}
    V^{\pi_\lambda}(s) \geq V^{\pi_\beta}(s) - R_{\max}\left(e^{2\epsilon/\lambda} - 1\right),
\end{equation}
where $\epsilon = \sigma\sqrt{2\log(2|\mathcal{A}|/\delta)}$. When $\lambda \geq 2\epsilon$, this simplifies to:
\begin{equation}
    V^{\pi_\lambda}(s) \geq V^{\pi_\beta}(s) - \frac{4R_{\max}\epsilon}{\lambda}.
\end{equation}

\begin{proof}
As before, let $\pi_\lambda^*(a|s) = \pi_\beta(a|s)\exp(r^*(s,a)/\lambda) / Z^*$ denote the exponential tilt under true rewards, and $\pi_\lambda(a|s) = \pi_\beta(a|s)\exp(\hat{r}(s,a)/\lambda) / Z$ the tilt under noisy rewards, where $Z^* = \sum_a \pi_\beta(a|s)\exp(r^*(s,a)/\lambda)$ and $Z = \sum_a \pi_\beta(a|s)\exp(\hat{r}(s,a)/\lambda)$ are the normalizing factors. As in Theorem~\ref{thm:noisy_improvement}, we condition on the event $|\xi(s,a)| \leq \epsilon$ for all $a \in \mathcal{A}$, which holds with probability at least $1 - \delta$ and fix the state $s \in \mathcal{S}$.

First, the value function can be decomposed as follows:
\begin{equation}\label{eq:decomp}
    V^{\pi_\lambda}(s) - V^{\pi_\beta}(s) = \underbrace{\left(V^{\pi_\lambda^*}(s) - V^{\pi_\beta}(s)\right)}_{\geq\, 0 \text{ by Proposition~\ref{prop:monotonic}}} - \underbrace{\left(V^{\pi_\lambda^*}(s) - V^{\pi_\lambda}(s)\right)}_{\text{noise difference}}.
\end{equation}
Now, we just need to bound the difference between the noisy and noise-free policy values. Since $r^*(s,a) \in [0, R_{\max}]$:
\begin{align}
    \left|V^{\pi_\lambda^*}(s) - V^{\pi_\lambda}(s)\right| 
    &= \left|\sum_a \left(\pi_\lambda^*(a|s) - \pi_\lambda(a|s)\right) r^*(s,a)\right| \nonumber \\
    &\leq R_{\max} \sum_a \left|\pi_\lambda^*(a|s) - \pi_\lambda(a|s)\right| \nonumber \\
    &= 2R_{\max}\,  D_{\mathrm{TV}}\!\left(\pi_\lambda^*(\cdot|s),\, \pi_\lambda(\cdot|s)\right). \label{eq:tv_reduction}
\end{align}

Define the noise weight $w(a) = \exp(\xi(s,a)/\lambda)$. Since $|\xi(s,a)| \leq \epsilon$:
\begin{equation}\label{eq:w_bounds}
    e^{-\epsilon/\lambda} \leq w(a) \leq e^{\epsilon/\lambda} \quad \forall\, a \in \mathcal{A}.
\end{equation}

Now we relate the noisy and clean policies. Since $\hat{r}(s,a) = r^*(s,a) + \xi(s,a)$, the noisy partition function factorizes as:
\begin{align}
    Z &= \sum_a \pi_\beta(a|s)\exp\!\left(\frac{\hat{r}(s,a)}{\lambda}\right) 
    = \sum_a \pi_\beta(a|s)\exp\!\left(\frac{r^*(s,a)}{\lambda}\right)\exp\!\left(\frac{\xi(s,a)}{\lambda}\right) \nonumber \\
    &= \sum_a \pi_\beta(a|s)\exp\!\left(\frac{r^*(s,a)}{\lambda}\right) w(a) 
    = Z^* \sum_a \pi_\lambda^*(a|s)\, w(a), \label{eq:Z_factor}
\end{align}
where the last step uses $\pi_\lambda^*(a|s) = \pi_\beta(a|s)\exp(r^*(s,a)/\lambda) / Z^*$. Dividing both sides by $Z^*$:
\begin{equation}\label{eq:Z_ratio}
    \frac{Z}{Z^*} = \sum_a \pi_\lambda^*(a|s)\, w(a),
\end{equation}

Therefore, the partition function ratio is simply the expected noise weight under the clean policy. Combined with \eqref{eq:w_bounds}, this gives $e^{-\epsilon/\lambda} \leq Z/Z^* \leq e^{\epsilon/\lambda}$, and therefore:
\begin{equation}\label{eq:rho_bounds}
    e^{-\epsilon/\lambda} \leq \rho \leq e^{\epsilon/\lambda}, \quad \text{where } \rho := Z^*/Z.
\end{equation}

For any action $a$, the policy ratio between the noisy and clean policies then follows directly:
\begin{equation}\label{eq:ratio}
    \frac{\pi_\lambda(a|s)}{\pi_\lambda^*(a|s)} = \frac{\pi_\beta(a|s)\exp(\hat{r}(s,a)/\lambda) / Z}{\pi_\beta(a|s)\exp(r^*(s,a)/\lambda) / Z^*} = w(a) \cdot \frac{Z^*}{Z}.
\end{equation}

Substituting \eqref{eq:ratio} into the total variation:
\begin{align}
    2\,D_{\mathrm{TV}}\!\left(\pi_\lambda^*,\, \pi_\lambda\right) 
    &= \sum_a \pi_\lambda^*(a|s)\left|w(a)\,\rho - 1\right|. \label{eq:tv_expr}
\end{align}
From \eqref{eq:w_bounds} and \eqref{eq:rho_bounds}, the product $w(a)\,\rho$ lies in $[e^{-2\epsilon/\lambda},\, e^{2\epsilon/\lambda}]$, so:
\begin{equation}
    |w(a)\,\rho - 1| \leq \max\!\left(e^{2\epsilon/\lambda} - 1,\; 1 - e^{-2\epsilon/\lambda}\right) = e^{2\epsilon/\lambda} - 1,
\end{equation}
where the equality holds because $e^x - 1 \geq 1 - e^{-x}$ for $x \geq 0$. Substituting back into \eqref{eq:tv_expr}:
\begin{equation}\label{eq:tv_bound}
    D_{\mathrm{TV}}\!\left(\pi_\lambda^*(\cdot|s),\, \pi_\lambda(\cdot|s)\right) \leq \frac{e^{2\epsilon/\lambda} - 1}{2}.
\end{equation}

Substituting \eqref{eq:tv_bound} into \eqref{eq:tv_reduction} and then into \eqref{eq:decomp}:
\begin{equation}
    V^{\pi_\lambda}(s) \geq V^{\pi_\beta}(s) - R_{\max}\!\left(e^{2\epsilon/\lambda} - 1\right).
\end{equation}

When $\lambda \geq 2\epsilon$, we have $2\epsilon/\lambda \leq 1$, and $e^x - 1 \leq 2x$ for $x \in [0,1]$, giving:
\begin{equation}
    V^{\pi_\lambda}(s) \geq V^{\pi_\beta}(s) - \frac{4R_{\max}\,\epsilon}{\lambda}.
\end{equation}
Substituting $\epsilon = \sigma\sqrt{2\log(2|\mathcal{A}|/\delta)}$ completes the proof.
\end{proof}

\section{Experiments}
\subsection{Data Distribution}
The following chart shows the rating distribution across the datasets:
\begin{figure}[H]
  \vskip 0.2in
  \begin{center}    \centerline{\includegraphics[width=3.5in]{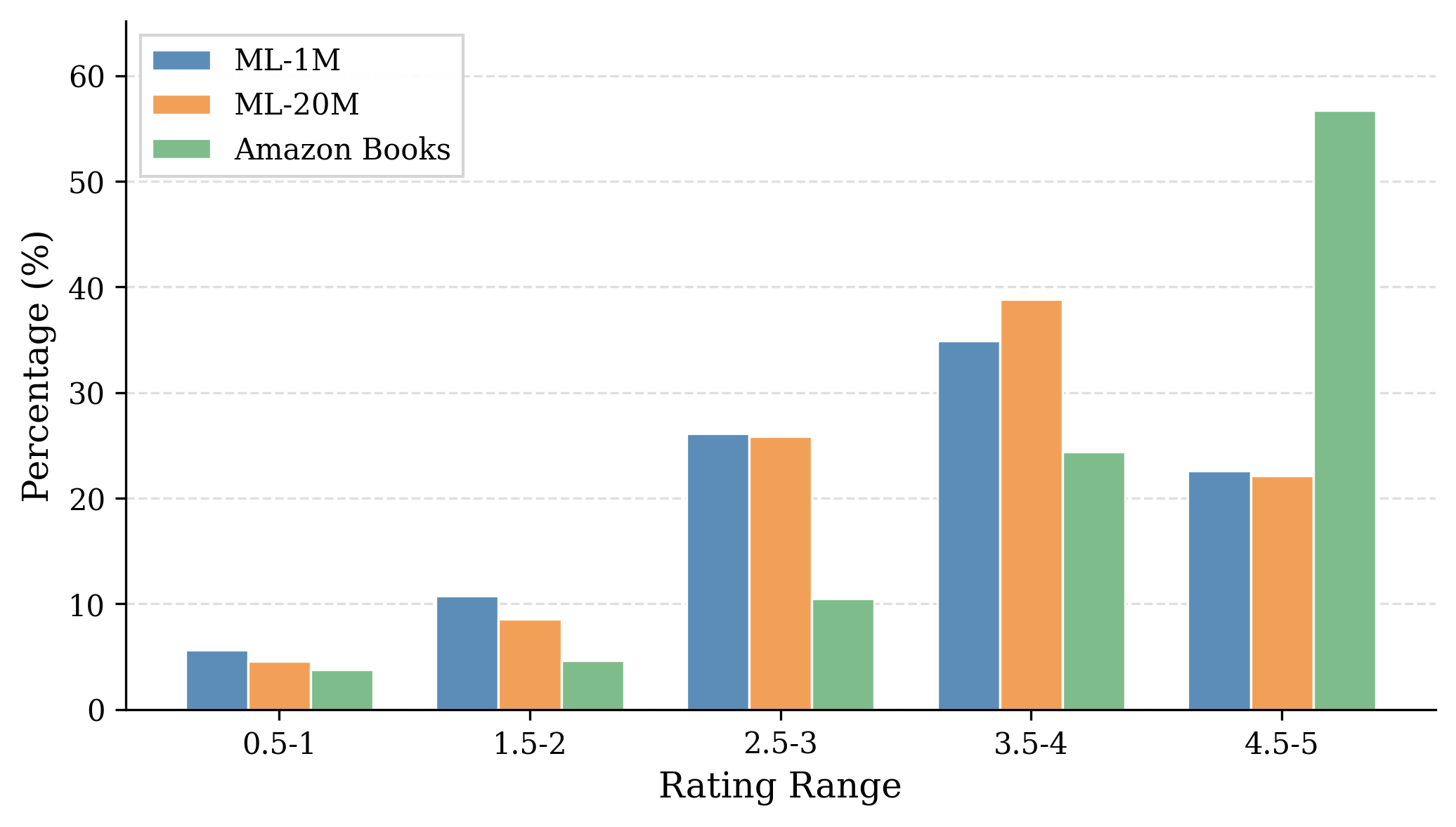}}
    \caption{
      Distribution of ratings in all three test datasets.
    }
    \label{fig:rating_dist}
  \end{center}
  \vskip -0.2in
\end{figure}
\afterpage{
\subsection{Additional Results} \label{app:metric_evolution}
\subsubsection{Metric Evolution for Exp-RSFT and RSFT}
These plots plot just Exp-RSFT and RSFT highlighting the superior performance of Exp-RSFT across all epochs for all datasets. The metrics at epoch 0 is the same as the metrics for the Behavior Cloning baseline.
\begin{figure*}[h]
  \vskip 0.2in
  \begin{center}    \centerline{\includegraphics[width=6.0in]{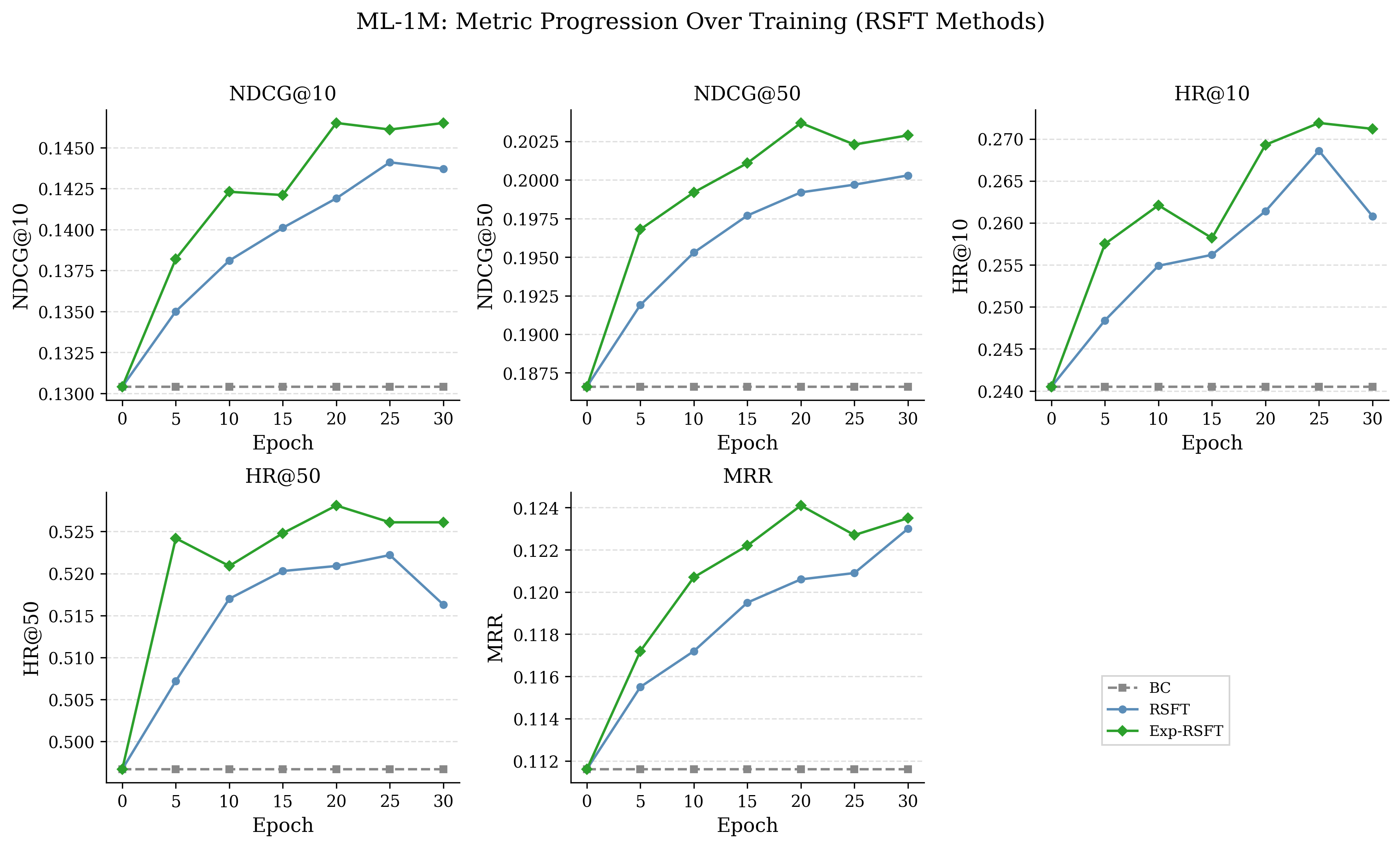}}
    \caption{
      Reward weighting versus exponential reward weighting for ML-1M.
    }
    \label{fig:metrics_zoom_ml-1m}
  \end{center}
  \vskip -0.2in
\end{figure*}
\begin{figure*}[h]
  \vskip 0.2in
  \begin{center}    \centerline{\includegraphics[width=6.0in]{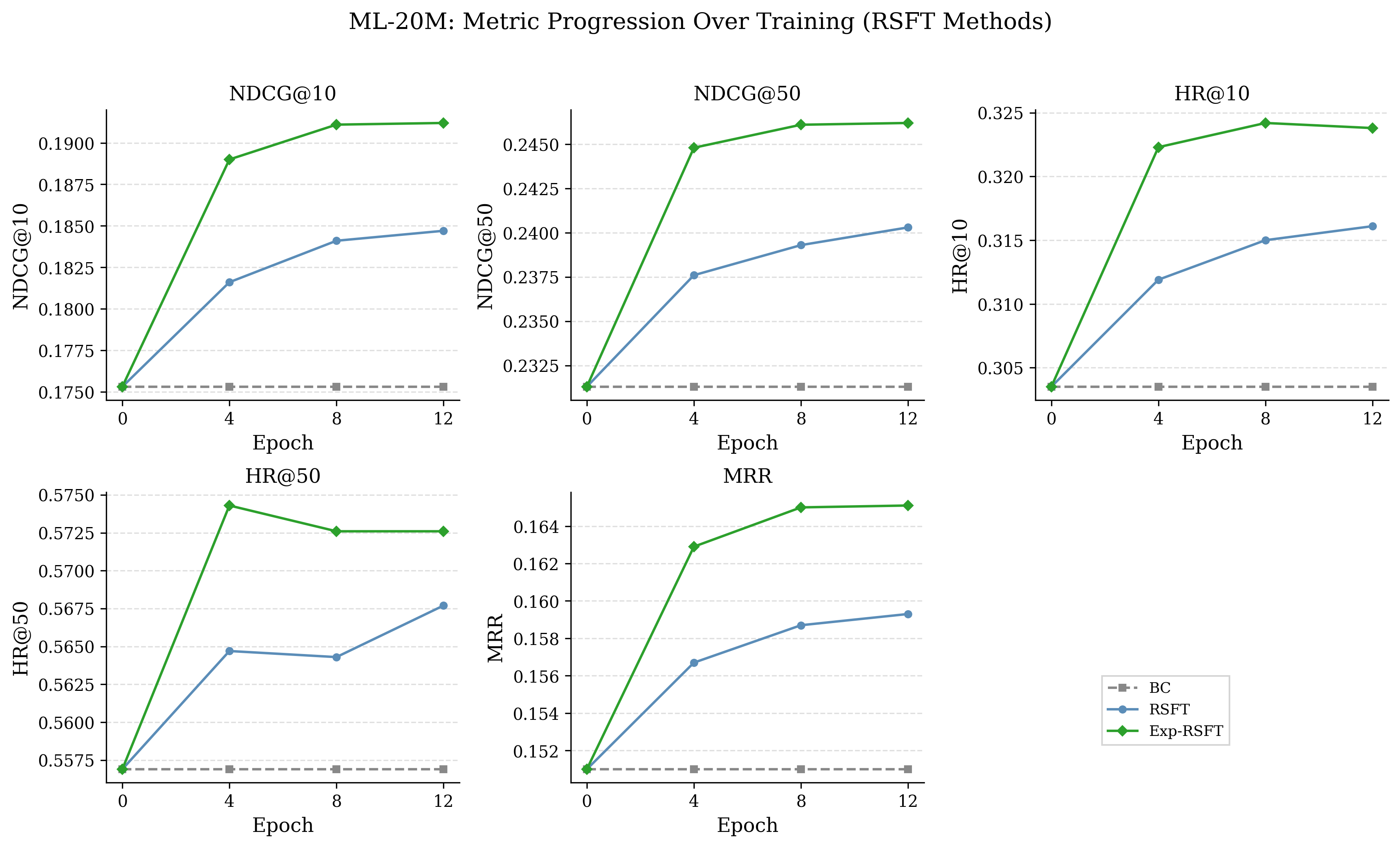}}
    \caption{
      Reward weighting versus exponential reward weighting for ML-20M.
    }
    \label{fig:metrics_zoom_ml-20m}
  \end{center}
  \vskip -0.2in
\end{figure*}
\begin{figure*}[h]
  \vskip 0.2in
  \begin{center}    \centerline{\includegraphics[width=6.0in]{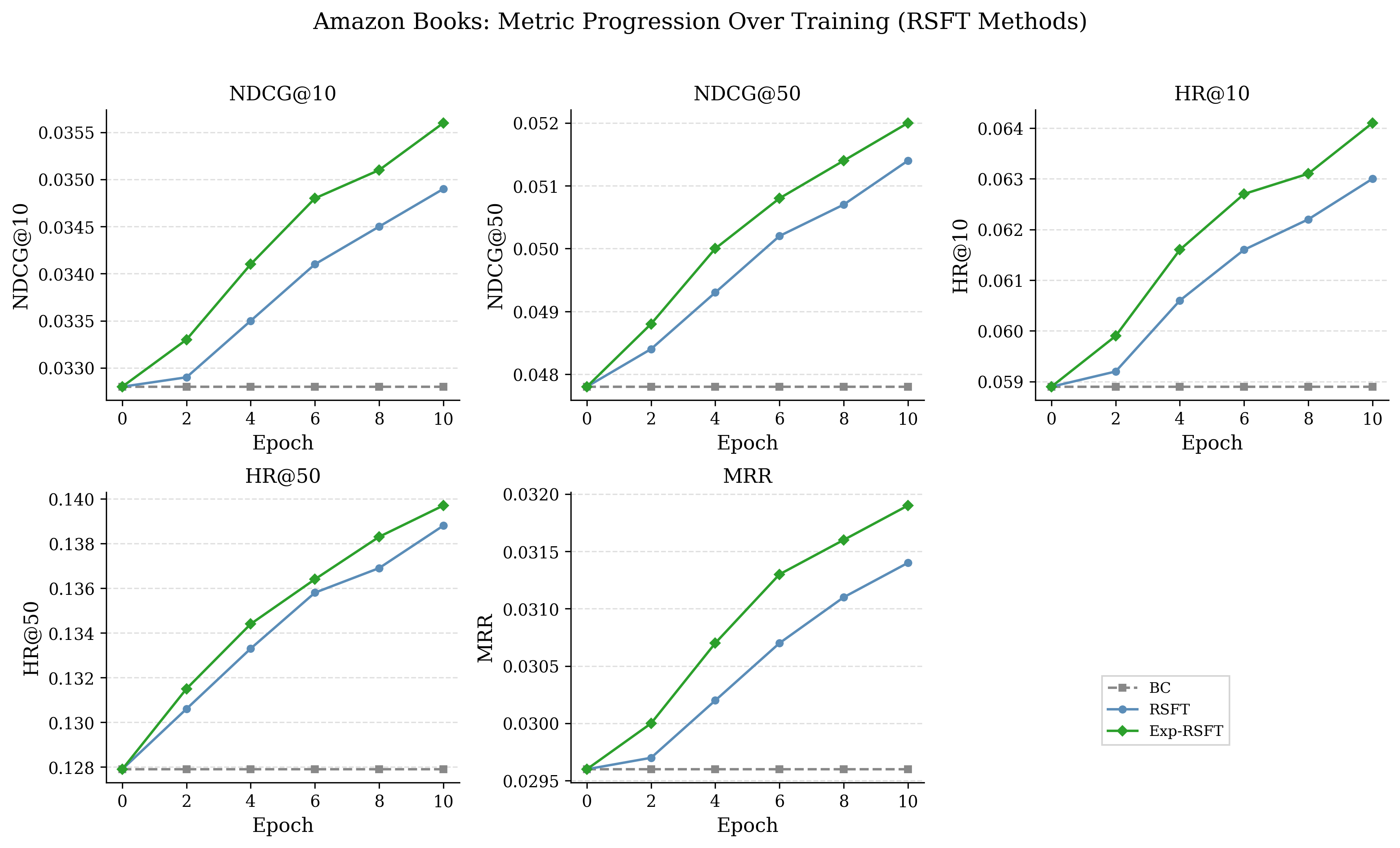}}
    \caption{
      Reward weighting versus exponential reward weighting for Amazon Books.
    }
    \label{fig:metrics_zoom_amazon-books}
  \end{center}
  \vskip -0.2in
\end{figure*}
\clearpage
}
\afterpage{
\subsubsection{Metric Evolution for All Algorithms}
These plots show the evolution of all metrics across all epochs for all datasets. Again, the metrics at epoch 0 is the same as the metrics for the Behavior Cloning baseline. We see that both PPO and DPO catastrophically collapse fairly early due to reward-model over-optimization.
\begin{figure*}[h]
  \vskip 0.2in
  \begin{center}    \centerline{\includegraphics[width=6.0in]{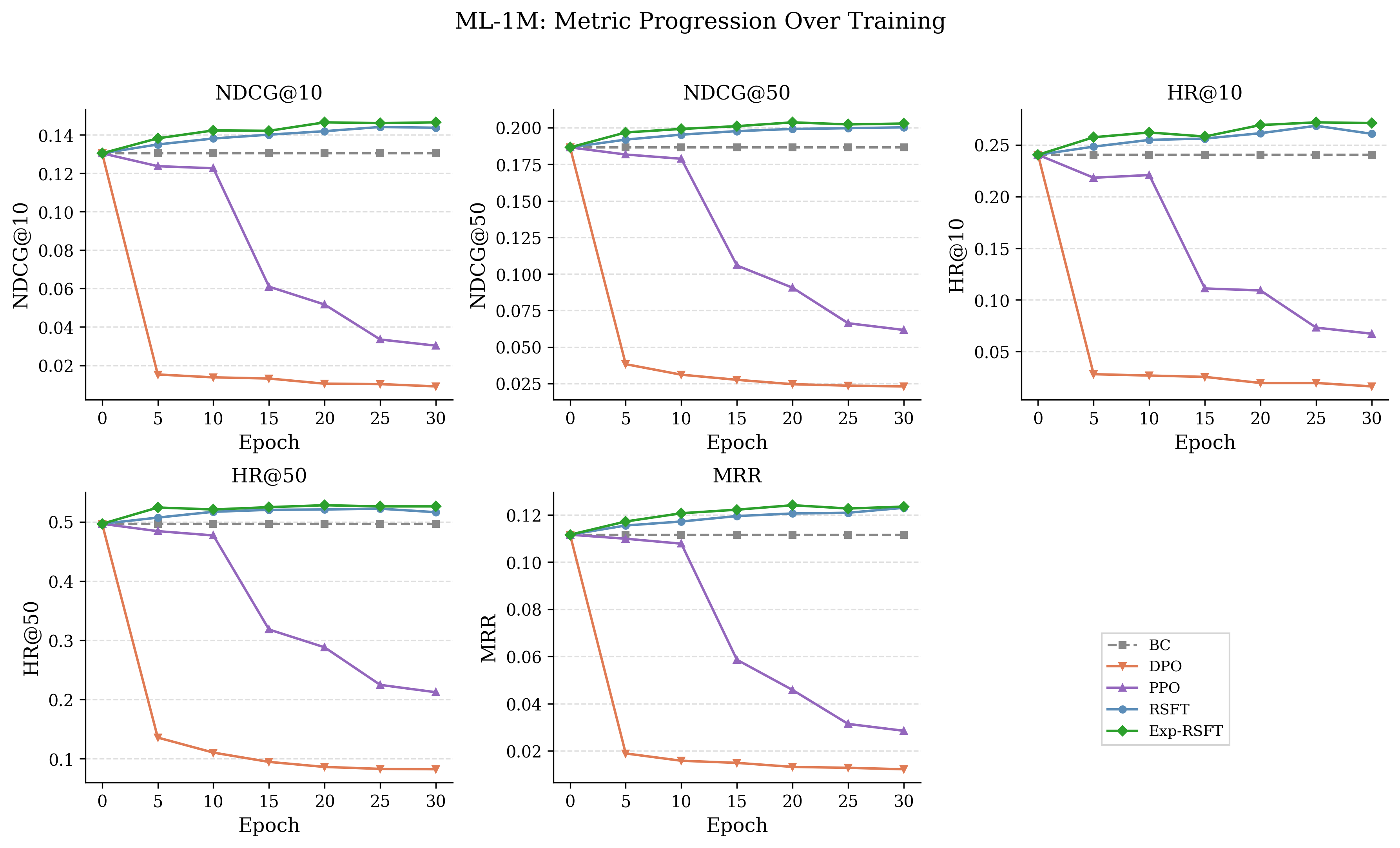}}
    \caption{
      All metrics for ML-1M.
    }
    \label{fig:metrics_all_ml-1m}
  \end{center}
  \vskip -0.2in
\end{figure*}
\begin{figure*}[h]
  \vskip 0.2in
  \begin{center}    \centerline{\includegraphics[width=6.0in]{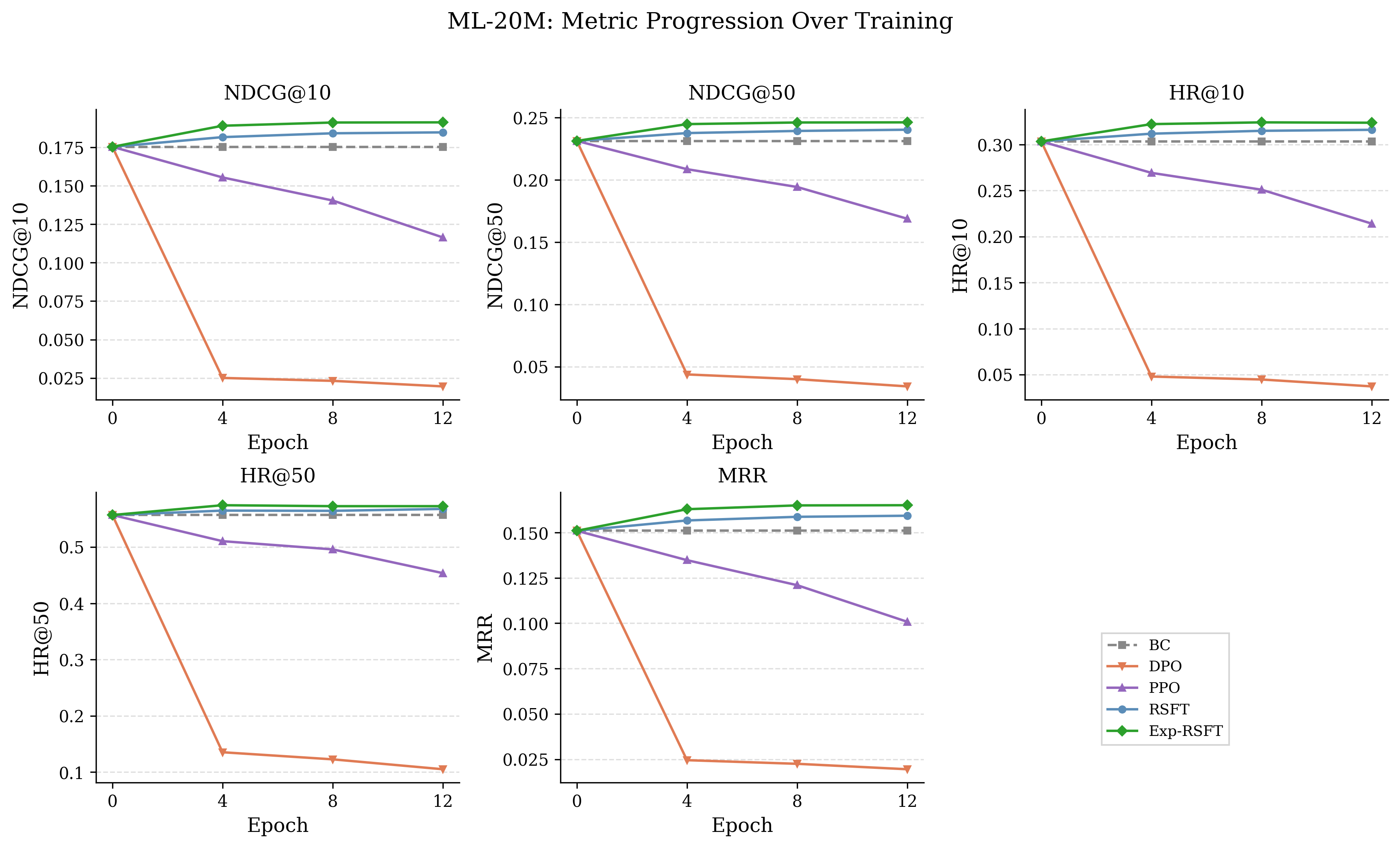}}
    \caption{
      All metrics for ML-20M.
    }
    \label{fig:metrics_all_ml-20m}
  \end{center}
  \vskip -0.2in
\end{figure*}
\begin{figure*}[h]
  \vskip 0.2in
  \begin{center}    \centerline{\includegraphics[width=6.0in]{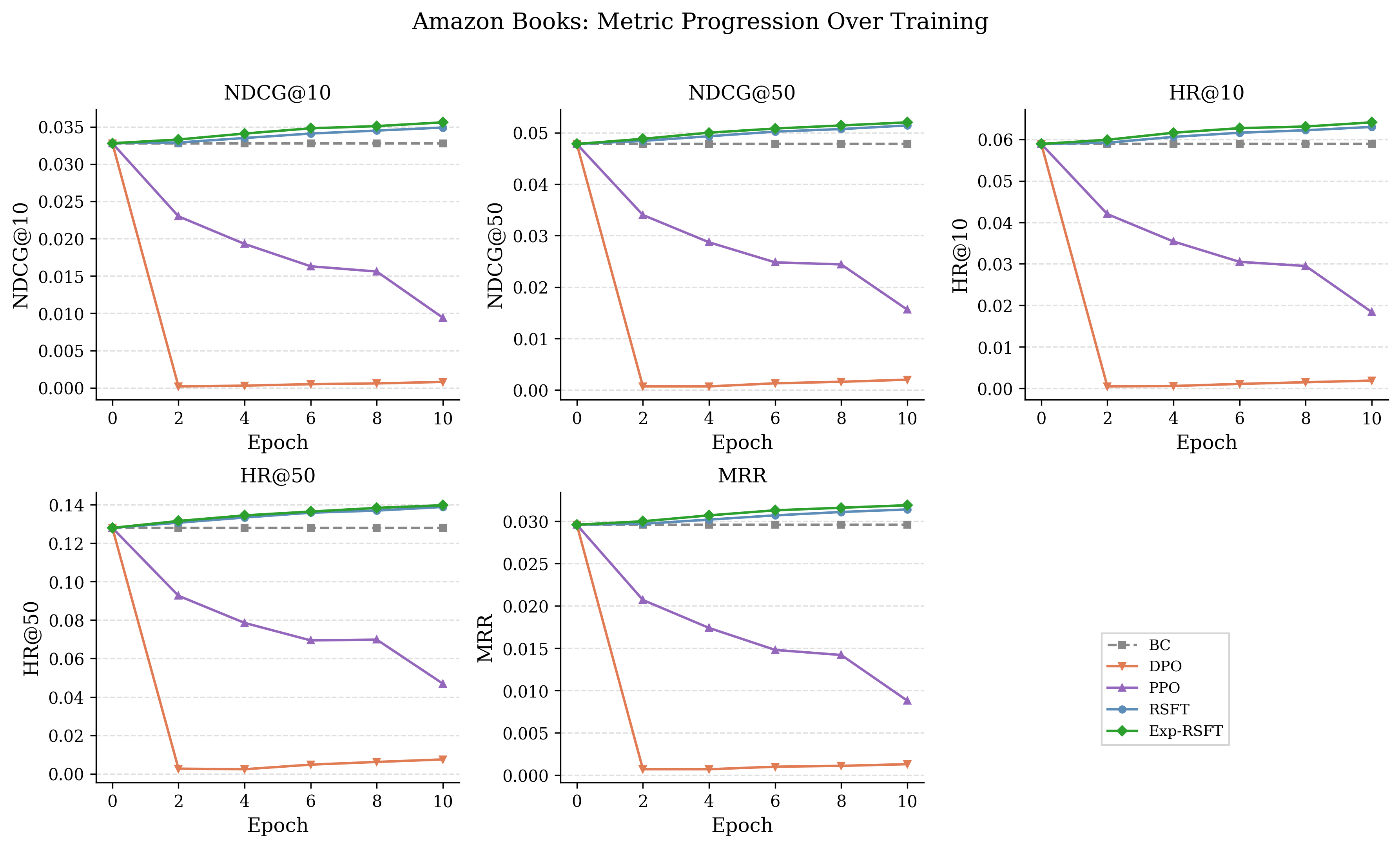}}
    \caption{
      All metrics for Amazon Books.
    }
    \label{fig:metrics_all_amazon-books}
  \end{center}
  \vskip -0.2in
\end{figure*}
}
\newpage
\afterpage{
\subsubsection{Performance-Robustness Trade-off}\label{app:lambda-sweep}
\begin{figure*}[h]
  \vskip 0.2in
  \begin{center}    \centerline{\includegraphics[width=6.0in]{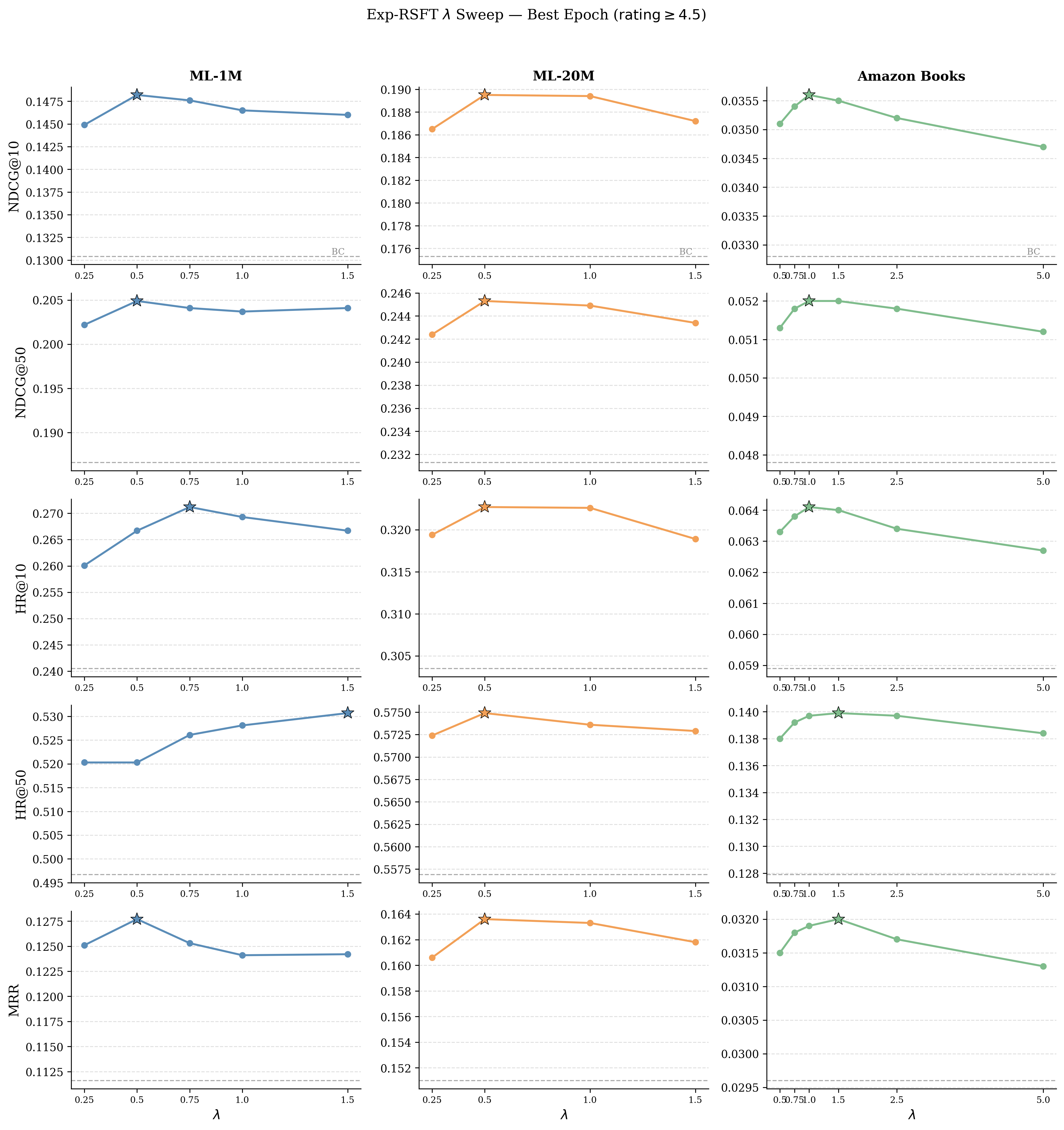}}
    \caption{
      Best metric values for all metrics for different $\lambda$.
    }
    \label{fig:lambda-sweep-best}
  \end{center}
  \vskip -0.2in
\end{figure*}
\begin{figure*}[ht]
  \vskip 0.2in
  \begin{center}    \centerline{\includegraphics[width=6.0in]{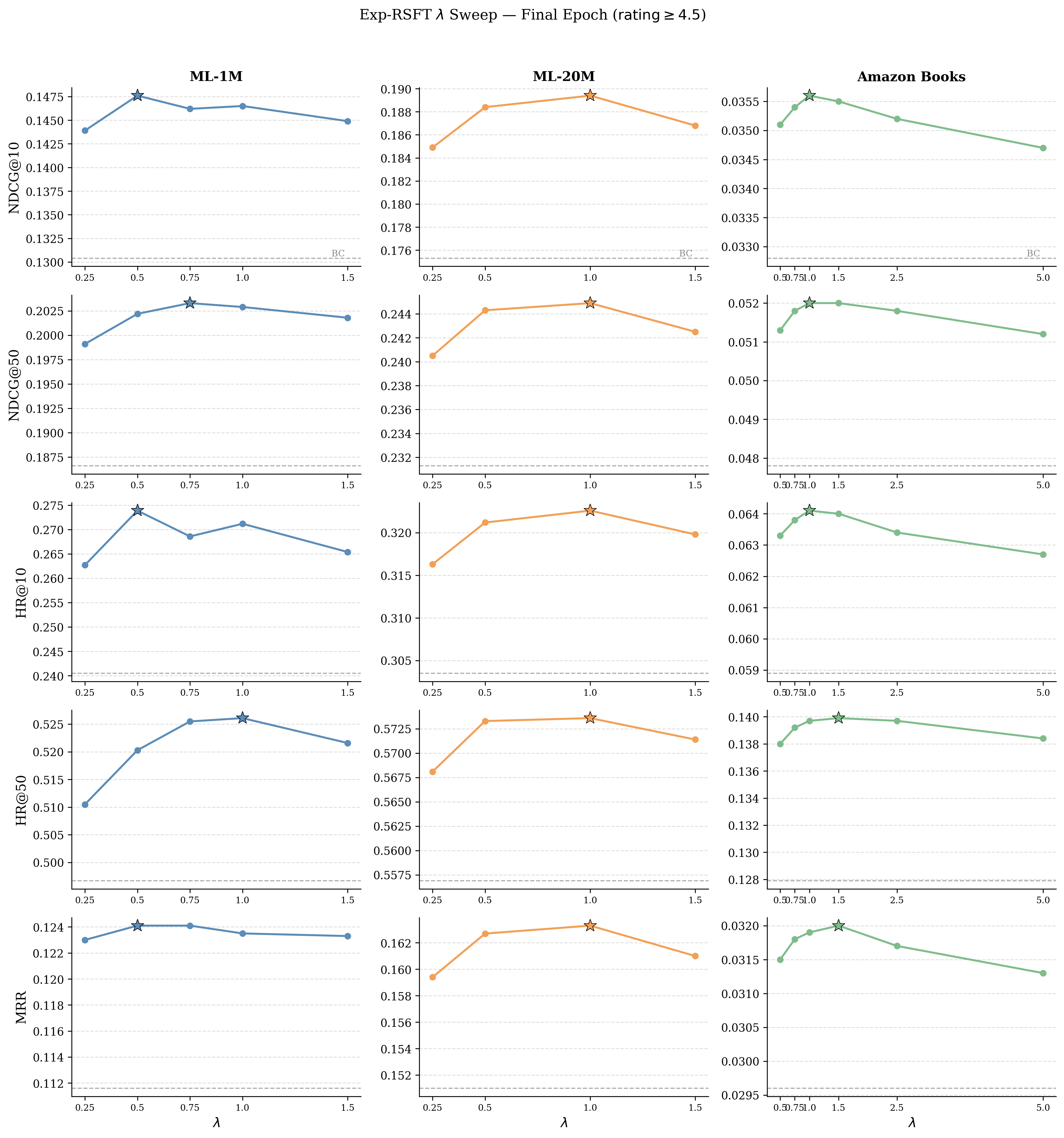}}
    \caption{
      Final epoch metric values for all metrics for different $\lambda$.
    }
    \label{fig:lambda-sweep-last}
  \end{center}
  \vskip -0.2in
\end{figure*}
}

\end{document}